\pdfoutput=1 
\documentclass[twoside]{article}

\usepackage{preprint}

\usepackage[round]{natbib}

\usepackage{microtype}
\usepackage{graphicx}
\usepackage{subcaption}
\usepackage{booktabs}
\usepackage{amsmath}
\usepackage{amssymb}
\usepackage{mathtools}
\usepackage{amsthm}
\usepackage{bm}
\usepackage{enumitem}

\DeclareMathOperator*{\argmax}{arg\,max}

\DeclareMathOperator*{\supp}{supp}
\DeclareMathOperator{\Tr}{tr}
\DeclareMathOperator{\range}{range}
\newcommand{\Cov}{\mathrm{Cov}}
\newcommand{\Var}{\mathrm{Var}}
\newcommand{\iu}{\mathrm{i}}
\newcommand{\indic}{\mathbf{1}}
\newcommand\norm[1]{\left\lVert#1\right\rVert}
\DeclarePairedDelimiter\ceil{\lceil}{\rceil}

\usepackage{hyperref}
\usepackage[capitalize,noabbrev]{cleveref}

\theoremstyle{plain}
\newtheorem{theorem}{Theorem}[section]
\newtheorem{proposition}[theorem]{Proposition}
\newtheorem{lemma}[theorem]{Lemma}
\newtheorem{corollary}[theorem]{Corollary}

\theoremstyle{definition}
\newtheorem{definition}[theorem]{Definition}
\newtheorem{assumption}[theorem]{Assumption}
\theoremstyle{remark}
\newtheorem{remark}[theorem]{Remark}

\begin{document}

\runningtitle{Asymptotic Performance of Time-Varying Bayesian Optimization}

\twocolumn[

\papertitle{Asymptotic Performance of Time-Varying Bayesian Optimization}

\paperauthor{ Anthony Bardou \And Patrick Thiran }

\paperaddress{ IC, EPFL \\ \texttt{anthony.bardou@epfl.ch} \And IC, EPFL \\ \texttt{patrick.thiran@epfl.ch} } ]

\begin{abstract}
  Time-Varying Bayesian Optimization~(TVBO) is the go-to framework for optimizing a time-varying black-box objective function that may be noisy and expensive to evaluate, but its excellent empirical performance remains to be understood theoretically. Is it possible for the instantaneous regret of a TVBO algorithm to vanish asymptotically, and if so, when? We answer this question of great importance by providing upper bounds and algorithm-independent lower bounds for the cumulative regret of TVBO algorithms. In doing so, we provide important insights about the TVBO framework and derive sufficient conditions for a TVBO algorithm to have the no-regret property. To the best of our knowledge, our analysis is the first to cover all major classes of stationary kernel functions used in practice.
\end{abstract}

\section{Introduction} \label{sec:intro}

Many real-world problems boil down to the optimization of a time-varying black-box function $f : \mathcal{S} \times \mathcal{T} \to \mathbb{R}$, where $\mathcal{S} \subset \mathbb{R}^d$ and $\mathcal{T} \subseteq \mathbb{R}$. Such time-varying problems arise when the objective function, which depends on the problem parameters $\bm x \in \mathcal{S}$, is also subjected to time-varying factors that cannot be controlled by the optimizer. Such a setting is common in online clustering~\citep{aggarwal2004framework}, management of unmanned aerial vehicles~\citep{melo2021dynamic} or network management~\citep{kim2019dynamic}.

Time-Varying Bayesian Optimization~(TVBO) is the natural extension of the Bayesian Optimization~(BO) framework to such settings. It relies on a Gaussian Process~(GP) prior whose covariance function combines a spatial kernel $k_S$ with a temporal kernel $k_T$~\citep{bogunovic2016time, nyikosa2018bayesian, bardou2024too, brunzemaevent}. Whereas BO is known to be no-regret for static objectives~\citep{gpucb}, the asymptotic performance of TVBO algorithms remains poorly understood. On the negative side, \citet{besbes2015non} proved that, among the objectives that do not become asymptotically static, there exists at least one on which any algorithm incurs a linear cumulative regret (i.e., a total optimization error that grows linearly with the number of observations, see Section~\ref{sec:background-tvbo}), and \citet{bogunovic2016time} established an algorithm-independent linear lower bound when $k_T$ is an exponential kernel. On the positive side, all sublinear upper bounds~\citep{zhou2021no, deng2022weighted, hong2023optimization, iwazaki2024near} require the variational budget of $f$, i.e., its cumulative variation across successive observation times, to grow sublinearly, i.e., $f$ to become asymptotically static. Whether, and when, the instantaneous regret of a TVBO algorithm can vanish for the kernels used in practice remains open.

\begin{figure*}[t]
  \centering
  \includegraphics[width=\textwidth]{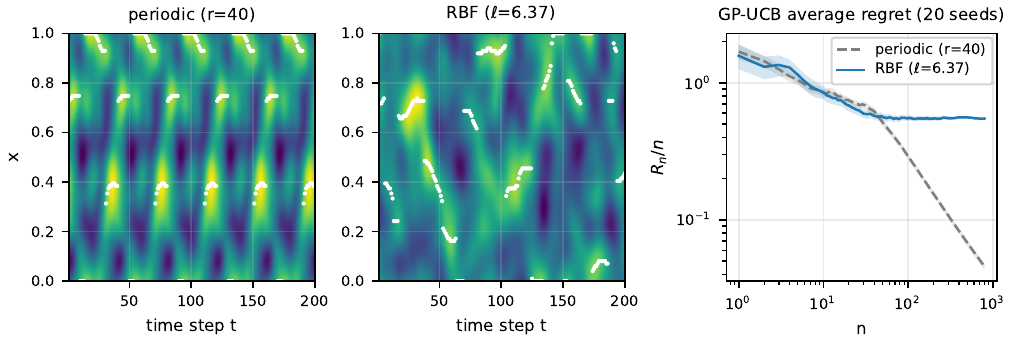}
  \caption{Same temporal one-step correlation, different behavior. (Left and center)~One sample of $f(\cdot, t)$ on $\mathcal{S} = [0, 1]$ for a periodic temporal kernel and for an RBF (i.e., squared exponential) temporal kernel with the same temporal one-step correlation $k_T(\Delta) \approx 0.988$, where $\Delta$ is the time between two consecutive observations. The maximizer $\bm x^*_t = \argmax_{\bm x \in \mathcal{S}} f(\bm x, t)$ is shown in white. Both objectives have a linearly growing variational budget. (Right)~Average regret $R_n/n$ (see Section~\ref{sec:background-tvbo}) of GP-UCB~\citep{gpucb} over 20 objectives sampled from each prior ($\pm$ one standard error): it vanishes for the periodic kernel, but plateaus for the RBF kernel.}
  \label{fig:teaser}
\end{figure*}

Consider the two objectives of Figure~\ref{fig:teaser}. Their temporal kernels have the same one-step correlation, so the objectives vary at the same speed, and both have a variational budget that grows linearly with time, so that none of the guarantees above applies to either of them. Yet, the GP-UCB algorithm~\citep{gpucb} appears to be no-regret on the first one and to plateau on the second one. In this paper, we show that what separates these two scenarios is not how fast the objective moves, but how much new information it reveals at each observation. The first objective is a superposition of rhythms whose amplitudes can be learned once and for all, whereas the second one keeps producing information that cannot be predicted from its past.

We make this intuition precise for all classes of stationary temporal kernels (see Table~\ref{tab:wrapup}). Our analysis relies on the classical prediction theory of stationary sequences~\citep{szego1920beitrage, kolmogorov1978stationary}, applied to the temporal process sampled at the observation frequency. It outlines the importance of a single quantity, the \emph{information rate} $\rho$ of the temporal kernel (see Definition~\ref{def:info_rate}), i.e., the classical mutual information rate between the sampled temporal process and its noisy observations~\citep{pinsker1964information}. More precisely, we make the following contributions:
\begin{enumerate}[label=(\roman*), leftmargin=*, itemsep=1pt, topsep=2pt]
    \item \textbf{Lower bound on the regret.} If the spectral density of $k_T$ (see Section~\ref{sec:background-spectral}) has a continuous part (equivalently, $\rho > 0$), every TVBO algorithm incurs $\mathbb{E}[R_n] \in \Theta(n)$, at every sampling frequency (Theorem~\ref{thm:lower_regret_bound}). This covers the broadband and band-limited kernels used in practice, and recovers the linear asymptotics of the lower bound of~\citet{bogunovic2016time} as a particular case. Interestingly, it also holds for band-limited objectives sampled above the Nyquist rate, which become perfectly predictable from a long enough noiseless past.
    \item \textbf{Lower bound on the information gain.} The maximum information gain $\gamma_n$ (see~\eqref{eq:info_gain}), which controls most regret analyses of GP bandits, satisfies $\gamma_n \geq \rho n$ (Proposition~\ref{prop:barrier}): information-gain-based regret analyses cannot be sublinear exactly when the lower bound applies.
    \item \textbf{Upper bound on the regret.} If the spectral density of $k_T$ is a mixture of Dirac deltas (almost-periodic and low-rank kernels), GP-UCB is no-regret under mild conditions, with rates as fast as $\tilde{\mathcal{O}}(\sqrt{n})$, where $\tilde{\mathcal{O}}$ hides logarithmic factors (Theorem~\ref{thm:upper_regret_bound} and Corollary~\ref{cor:rates}). To the best of our knowledge, this is the first sublinear regret bound for TVBO that does not require the objective to become asymptotically static.
    \item \textbf{Experiments.} We confirm each result on objectives sampled from the prior. We also observe that the regret of an idealized oracle is approximately proportional to the asymptotic prediction error $\sigma^2_0(e^{2\rho} - 1)$ of the temporal process, where $\sigma^2_0$ is the observational noise, across kernel families (Section~\ref{sec:experiments}).
\end{enumerate}
Table~\ref{tab:wrapup} summarizes our results.

\begin{table*}[t]
  \caption{Properties of the four classes of stationary temporal kernels $k_T$ according to their spectral densities $S_T$ (see Section~\ref{sec:background-spectral}). For each class, the table reports the properties of $S_T$, an example of a temporal kernel from this class and the support of its spectral density, its information rate $\rho$ (Definition~\ref{def:info_rate}), and the guarantees about the cumulative regret $R_n$ provided by Theorems~\ref{thm:lower_regret_bound} and~\ref{thm:upper_regret_bound}. The lower bounds hold for every TVBO algorithm and every sampling frequency. The upper bounds hold with high probability~(w.h.p.) for GP-UCB only, under mild conditions on rythmic kernels. Mixtures whose spectral density has a continuous part (e.g., $0.95$ periodic $+\ 0.05$ RBF) behave like the first two classes.}
  \label{tab:wrapup}
  \centering
  \begin{tabular}{lcccc}
    \toprule
    & \multicolumn{4}{c}{\textbf{Temporal Kernel Class}}\\
    \cmidrule(r){2-5}
     & Broadband & Band-Limited & Almost-Periodic & Low-Rank \\
    \midrule
    Continuous part of $S_T$ & Yes & Yes & No & No\\
    Bounded $\supp(S_T)$ & No & Yes & Not necessarily & Yes\\
    Dirac deltas in $S_T$ & None & None & Infinitely many & Finitely many\\
    \midrule
    Example of $k_T$ & RBF & Sinc($\tau$) & Periodic($r$) & Sum of $L$ Cosines \\
    $\supp(S_T)$ & $\mathbb{R}$ & $[-\tau, \tau]$ & $\left\{p/r\right\}_{p \in \mathbb{Z}}$ & $\{0\} \cup \left\{\pm\omega_p\right\}_{p \in [L]}$\\
    \midrule
    Information rate $\rho$ & $> 0$ & $> 0$ & $0$ & $0$\\
    Guarantees on $R_n$ & $\mathbb{E}[R_n] \in \Theta(n)$ & $\mathbb{E}[R_n] \in \Theta(n)$ & $R_n \in o(n)$ w.h.p. & $R_n \in \tilde{\mathcal{O}}(\sqrt{(2L+1)n})$ w.h.p.\\
  \bottomrule
  \end{tabular}
\end{table*}

\section{Background and Core Assumptions}

\subsection{Time-Varying Bayesian Optimization} \label{sec:background-tvbo}

\paragraph{Surrogate Model.} The goal of a TVBO algorithm is to optimize a time-varying black box $f : \mathcal{S} \times \mathcal{T} \to \mathbb{R}$, where $\mathcal{S} \subset \mathbb{R}^d$ is a compact problem parameter space (i.e., the spatial domain) and $\mathcal{T} \subseteq \mathbb{R}$ is the temporal domain. It assumes that $f$ is a $\mathcal{GP}\left(0, k\right)$ whose mean is zero without loss of generality~(w.l.o.g.) and whose covariance function is $k : \left(\mathcal{S} \times \mathcal{T}\right)^2 \to \mathbb{R}$. Given a dataset $\mathcal{D}_n = \left\{(\bm x_i, t_i, y_i)\right\}_{i \in [n]}$ of $n$ observations, with $[n] = \{1, \cdots, n\}$, where $y_i = f(\bm x_i, t_i) + \epsilon_i$ and $\epsilon_i \sim \mathcal{N}\left(0, \sigma^2_0\right)$, the posterior of $f$ is a GP with mean
\begin{equation}
    \mu_n(\bm x, t) = k^\top((\bm x, t), \mathcal{D}_n) \left(\bm K^{(n)} + \sigma^2_0 \bm I\right)^{-1} \bm y_n, \label{eq:posterior_mean}
\end{equation}
and covariance
\begin{equation}
    \begin{split}\label{eq:posterior_cov}
        \Cov_n((\bm x, t), (\bm x', t')) &= k((\bm x, t), (\bm x', t')) - k^\top((\bm x, t), \mathcal{D}_n) \\&\left(\bm K^{(n)} + \sigma^2_0 \bm I\right)^{-1} k((\bm x', t'), \mathcal{D}_n),
    \end{split}
\end{equation}
where $k(\mathcal{X}, \mathcal{Y}) = \left(k(\bm z, \bm z')\right)_{\bm z \in \mathcal{X}, \bm z' \in \mathcal{Y}}$ for any finite sets $\mathcal{X}, \mathcal{Y} \subset \mathcal{S} \times \mathcal{T}$ (the dataset $\mathcal{D}_n$ being identified with its inputs $\{(\bm x_i, t_i)\}_{i \in [n]}$), $\bm K^{(n)} = k(\mathcal{D}_n, \mathcal{D}_n)$, $\bm y_n = \left(y_1, \cdots, y_n\right)$ and $\bm I$ is the $n \times n$ identity matrix. We also denote the posterior variance by $\sigma^2_n(\bm x, t) = \Cov_n((\bm x, t), (\bm x, t))$.

\paragraph{Acquisition Function.} At time $t_{n+1}$, a TVBO algorithm queries $\bm x_{n+1} = \argmax_{\bm x \in \mathcal{S}} \varphi_n(\bm x, t_{n+1})$, where the acquisition function $\varphi_n$ trades off exploration and exploitation. For instance, the GP Upper Confidence Bound~(GP-UCB) algorithm~\citep{gpucb} uses $\varphi_n = \mu_n + \beta_{n+1}^{1/2} \sigma_n$, where $\beta_{n+1} > 0$ controls the amount of exploration. For our lower bound, we call \emph{TVBO algorithm} any rule that selects $\bm x_{n+1}$ as a measurable function of $\mathcal{D}_n$ and of an internal random variable $U$ independent of $f$ and of the observational noise.\footnote{Accounting for $U$ allows us to offer guarantees for stochastic search algorithms like random search.}

\paragraph{Asymptotic Performance.} The optimization error at time $t_i$ is measured by the instantaneous regret $r_i = \sup_{\bm x \in \mathcal{S}} f(\bm x, t_i) - f(\bm x_i, t_i)$, which is aggregated into the cumulative regret $R_n = \sum_{i = 1}^n r_i$. An algorithm has the no-regret property if $\lim_{n \to \infty} R_n/n = 0$. Most regret analyses of GP bandits~\citep{gpucb, bogunovic2016time, brunzemaevent} control $R_n$ through the maximum information gain
\begin{equation} \label{eq:info_gain}
    \gamma_n(k) = \max_{\bm x_1, \cdots, \bm x_n \in \mathcal{S}} \frac{1}{2} \log\det\left(\bm I + \sigma_0^{-2}\bm K^{(n)}\right),
\end{equation}
where $\bm K^{(n)}$ is computed at the inputs $\{(\bm x_i, t_i)\}_{i \in [n]}$ and the time components $t_1, \cdots, t_n$ are fixed (see Assumption~\ref{ass:sampling_freq}). It is the largest mutual information between $f$ and $n$ noisy observations of $f$, measured in nats (i.e., with the natural logarithm). For a kernel $k'$ on $\mathcal{S}$ alone (e.g., $k_S$), $\gamma_n(k')$ is defined similarly, with $\bm K^{(n)} = \left(k'(\bm x_i, \bm x_j)\right)_{i, j \in [n]}$.

\subsection{Core Assumptions} \label{sec:background-assumptions}

All the results of this paper rely on Assumptions~\ref{ass:surrogate}-\ref{ass:lipschitz}, which are standard in TVBO~\citep{bogunovic2016time, nyikosa2018bayesian, bardou2024too, brunzemaevent}. Assumption~\ref{ass:surrogate} justifies the Bayesian setting. Assumption~\ref{ass:covariance} is a simple, popular and powerful way to encode spatio-temporal dynamics with two correlation functions, $k_S$ and $k_T$, dedicated to spatial and temporal dynamics, respectively. Its condition on $k_S$ only excludes degenerate spatial kernels (e.g., constant ones), and its condition on $S_T$ is satisfied by all the temporal kernels used in practice, as well as by their products and mixtures (see Appendix~\ref{app:examples}). Assumption~\ref{ass:sampling_freq} ensures that observations are collected at a fixed sampling frequency $1/\Delta$ and is often implicitly made in TVBO papers. Finally, Assumption~\ref{ass:lipschitz} ensures that the GP is not too erratic in the spatial domain~\citep{gpucb}. It is satisfied when $k_S$ is an RBF kernel or a Matérn kernel with smoothness $\nu > 2$.

\begin{assumption}[Surrogate Model] \label{ass:surrogate}
    The time-varying black box $f : \mathcal{S} \times \mathcal{T} \to \mathbb{R}$ is a $\mathcal{GP}(0, k)$, where $\mathcal{S} = [0, 1]^d$ w.l.o.g.
\end{assumption}

\begin{assumption}[Covariance Function] \label{ass:covariance}
    The covariance function $k$ admits the decomposition
    \begin{equation} \label{eq:decomposable_covariance}
        k((\bm x, t), (\bm x', t')) = k_S(\bm x, \bm x') k_T(t - t'),
    \end{equation}
    where $k_S : \mathcal{S} \times \mathcal{S} \to [-1, 1]$ (resp., $k_T : \mathbb{R} \to [-1, 1]$) is a stationary correlation function, i.e., $k_S(\bm x, \bm x) = k_T(0) = 1$. Furthermore, $k_S$ is continuous and $k_S(\bm x, \bm x') < 1$ for all $\bm x \neq \bm x'$. Finally, $k_T(t) = \int_{\mathbb{R}} S_T(\omega) e^{2\pi \iu \omega t} d\omega$ for all $t \in \mathbb{R}$, where the spectral density $S_T$ of $k_T$ (see Section~\ref{sec:background-spectral}) is of the form
    \begin{equation} \label{eq:spectral_density_form}
        \begin{split}
            S_T(\omega) = S_T^c(\omega) &+ \alpha_0\delta(\omega) \\ &+ \sum_{p \geq 1} \alpha_p\left(\delta(\omega - \omega_p) + \delta(\omega + \omega_p)\right),
        \end{split}
    \end{equation}
    where $\iu$ is the imaginary unit, $\omega$ is the frequency (measured in cycles per unit of time) and $\delta$ is the Dirac delta. The continuous part $S_T^c : \mathbb{R} \to \mathbb{R}_{\geq 0}$ of $S_T$ is even and integrable, and its discrete part consists of Dirac deltas with weights $\alpha_0, \alpha_1, \cdots \geq 0$ at the frequencies $0$ and $\pm\omega_1, \pm\omega_2, \cdots$, where $\omega_1, \omega_2, \cdots > 0$ are distinct.
\end{assumption}

\begin{assumption}[Sampling Frequency of Observations] \label{ass:sampling_freq}
    Observations are sampled at a fixed frequency $1/\Delta$, with $0 < \Delta < +\infty$. Consequently, the time component of the $i$th observation $\left(\bm x_i, t_i, y_i\right)$ is $t_i = i\Delta$.
\end{assumption}

\begin{assumption}[Lipschitzness in Space] \label{ass:lipschitz}
    Let $g \sim \mathcal{GP}(0, k_S)$. There exist $a \geq 1$ and $b > 0$ such that, for any $u > 0$ and any $j \in [d]$,
    \begin{equation*}
        \mathbb{P}\left[\sup_{\bm x \in \mathcal{S}}\left|\frac{\partial g(\bm x)}{\partial x_j}\right| > u\right] \leq a e^{-(u/b)^2}.
    \end{equation*}
\end{assumption}

\subsection{Spectral View of the Temporal Kernel} \label{sec:background-spectral}

By Bochner's theorem~\citep{bochner2005harmonic}, the stationary correlation function $k_T$ is the Fourier transform of a nonnegative and even spectral density $S_T$ with $\int_{\mathbb{R}} S_T(\omega) d\omega = k_T(0) = 1$, that is,
\begin{equation} \label{eq:bochner}
    k_T(t) = \int_{\mathbb{R}} S_T(\omega) e^{2\pi \iu \omega t} d\omega, \quad t \in \mathbb{R},
\end{equation}
where $S_T$ may contain Dirac deltas. As in Assumption~\ref{ass:covariance}, $S_T^c$ is the continuous part of $S_T$ and the sum of Dirac deltas in~\eqref{eq:spectral_density_form} is its discrete part. We write $S_T^c \neq 0$ when $\int_{\mathbb{R}} S_T^c(\omega) d\omega > 0$. As described in Table~\ref{tab:wrapup}, the kernels used in practice fall into four classes, depending on the boundedness of the support $\supp(S_T)$ of $S_T$ and on the number of Dirac deltas in $S_T$ (see Appendix~\ref{app:examples}). Broadband kernels (e.g., RBF, Matérn or rational quadratic) and band-limited kernels with cutoff frequency $\tau$ (e.g., sinc and sinc$^2$) have no discrete part, and $\supp(S_T)$ is $\mathbb{R}$ and $[-\tau, \tau]$, respectively. Almost-periodic kernels~\citep{bohr1926theorie} (e.g., the periodic kernel of~\citet{mackay1998introduction}) and low-rank kernels (e.g., sums of cosines and random Fourier features~\citep{rahimi2007random}) have no continuous part, and their spectral densities are mixtures of infinitely and finitely many Dirac deltas, respectively.

As we show in Sections~\ref{sec:lower_bound} and~\ref{sec:upper_bound}, the presence of a continuous part in $S_T$ is decisive: it forces every TVBO algorithm to have a linear regret, but GP-UCB can be no-regret otherwise. Therefore, we give a name to the two classes whose spectral densities have no continuous part.

\begin{definition}[Rhythmic kernel] \label{def:rhythmic}
    A temporal kernel $k_T$ is rhythmic if it is almost-periodic or low-rank, i.e., if $S_T^c = 0$. Equivalently, $\alpha_0 + 2\sum_{p \geq 1} \alpha_p = 1$ and
    \begin{equation} \label{eq:rhythmic}
        k_T(t) = \alpha_0 + 2\sum_{p \geq 1} \alpha_p \cos(2\pi \omega_p t).
    \end{equation}
\end{definition}

A rhythmic kernel describes temporal dynamics that are a superposition of oscillations with fixed frequencies $\omega_p$ (e.g., daily or weekly cycles). As shown by Lemma~\ref{lem:szego} below, such dynamics can be predicted with an arbitrarily small error given a long enough dataset of observations, even when they are noisy.

Because observations are collected at times $t_i = i\Delta$ (Assumption~\ref{ass:sampling_freq}), a TVBO algorithm only experiences the temporal correlations through the samples $\left(k_T(h\Delta)\right)_{h \in \mathbb{Z}}$. For any fixed $\bm x \in \mathcal{S}$, they are the autocovariances of the sequence $\left(f(\bm x, i\Delta)\right)_{i \in \mathbb{Z}}$, since $k_S(\bm x, \bm x) = 1$. The spectral density of this sequence is obtained by folding $S_T$ (aliasing, see Lemma~\ref{lem:sampled}). Its discrete part consists of Dirac deltas located at the folded frequencies $2\pi\Delta\omega_p$ modulo $2\pi$, and its continuous part is, for any angular frequency $\theta$ (in radians per observation),
\begin{equation} \label{eq:sampled_density}
    s_\Delta(\theta) = \frac{1}{\Delta} \sum_{\ell \in \mathbb{Z}} S_T^c\left(\frac{\theta + 2\pi\ell}{2\pi\Delta}\right), \quad \theta \in (-\pi, \pi].
\end{equation}

To derive asymptotic regret bounds in this setting, one obvious question to address is: in the long run, how much new information does each noisy observation bring about this sequence? For a stationary Gaussian sequence observed in white noise, the answer is well known and given by the mutual information rate between the sequence and its noisy observations. It is a classical quantity from information theory that only depends on the spectral density of the sequence~\citep{pinsker1964information}, and it generalizes Shannon's formula $\frac{1}{2}\log\left(1 + \text{signal-to-noise ratio}\right)$ to correlated signals~\citep{shannon1948mathematical}. We apply it to the samples of the temporal kernel $\left(k_T(h\Delta)\right)_{h \in \mathbb{Z}}$, which yields the central quantity of this paper.
\begin{definition}[Information rate of $k_T$] \label{def:info_rate}
    The information rate of $k_T$ at sampling frequency $1/\Delta$ and noise level $\sigma^2_0$, in nats per observation, is
    \begin{equation} \label{eq:info_rate}
        \rho = \frac{1}{4\pi} \int_{-\pi}^{\pi} \log\left(1 + \frac{s_\Delta(\theta)}{\sigma^2_0}\right) d\theta.
    \end{equation}
\end{definition}

Only the continuous part $s_\Delta$ of the spectral density enters~\eqref{eq:info_rate}: the oscillations described by the Dirac deltas of $S_T$ bring no new information asymptotically, because they can be learnt from the past once and for all. Since $\frac{1}{2\pi}\int_{-\pi}^{\pi} s_\Delta = \int_{\mathbb{R}} S_T^c \leq 1$ (Lemma~\ref{lem:sampled}), Jensen's inequality gives $0 \leq \rho \leq \frac{1}{2}\log(1 + \sigma_0^{-2})$. The upper bound is the mutual information between a $\mathcal{N}(0, 1)$ variable and its noisy observation, i.e., the information brought by an observation whose signal is independent of the past. The bound is tight, since it can be attained by, e.g., a sinc kernel with cutoff frequency $\tau = 1/(2\Delta)$, whose samples $f(\bm x, i\Delta)$ are i.i.d.

The following result, proven in Appendix~\ref{app:szego}, relates $\rho$ to the predictability of the temporal process. It is a direct consequence of the Kolmogorov--Szeg\H{o} formula~\citep{szego1920beitrage, kolmogorov1978stationary}.

\begin{lemma} \label{lem:szego}
    Let $g \sim \mathcal{GP}(0, k_T)$ and, for all $i \in \mathbb{Z}$, $X_i = g(i\Delta)$ and $Y_i = X_i + \epsilon_i$, where the $\epsilon_i \sim \mathcal{N}(0, \sigma^2_0)$ are i.i.d. and independent of $g$. Then, (i)~$\rho > 0$ if and only if $S_T^c \neq 0$, for every $\Delta > 0$, and (ii)~for all $m \in \mathbb{N}$, $e_m = \Var\left[X_0 | Y_{-1}, \cdots, Y_{-m}\right]$ is non-increasing in $m$ and
    \begin{equation} \label{eq:innovation}
        \lim_{m \to \infty} e_m = e_\infty \coloneqq \sigma^2_0\left(e^{2\rho} - 1\right).
    \end{equation}
\end{lemma}

In other words, $e_\infty$ is the part of $X_0$ that cannot be predicted from the whole noisy past, and $\rho$ is the mutual information rate between $(X_i)_{i \in \mathbb{Z}}$ and $(Y_i)_{i \in \mathbb{Z}}$. Lemma~\ref{lem:szego} states that this irreducible uncertainty is positive as soon as $S_T$ has a continuous part, whatever the sampling frequency, and that it vanishes for rhythmic kernels.

\section{Fresh Information Forces Linear Regret} \label{sec:lower_bound}

We now show that a positive information rate $\rho$ prevents every TVBO algorithm from being no-regret. For a function $g : \mathcal{S} \to \mathbb{R}$, let $\range(g) = \sup_{\mathcal{S}} g - \inf_{\mathcal{S}} g$, and let $\Phi$ and $\phi$ denote the cumulative distribution function and the probability density function of $\mathcal{N}(0, 1)$.

\begin{theorem} \label{thm:lower_regret_bound}
    Let $S_T^c \neq 0$. Let $g_S \sim \mathcal{GP}(0, k_S)$, $\bar R = \mathbb{E}\left[\range(g_S)\right]$, $c_S = \min_{\norm{\bm u - \bm v}_2 \geq 1/2} \left(1 - k_S(\bm u, \bm v)\right)$ and $\sigma^2_\star = 2 c_S e_\infty$. Then, for any TVBO algorithm and any $n \in \mathbb{N}^*$,
    \begin{equation} \label{eq:lower_instantaneous}
        \mathbb{E}\left[r_n\right] \geq c \coloneqq \frac{\sigma_\star}{2}\, \Psi\left(-\frac{2\bar R}{\sigma_\star}\right) > 0,
    \end{equation}
    where $\Psi(z) = z\Phi(z) + \phi(z)$. Consequently, $\mathbb{E}\left[R_n\right] \in \Theta(n)$ and $\mathbb{P}\left[R_n \geq cn/2\right] \geq c^2 / \left(4\mathbb{E}\left[\range(g_S)^2\right]\right)$ for all $n \in \mathbb{N}^*$.
\end{theorem}

Theorem~\ref{thm:lower_regret_bound} is proven in Appendix~\ref{app:lower_bound}. The proof proceeds in three steps. (i)~Regardless of the TVBO algorithm, $r_n$ is at least the positive part of the gap $f(\bm x'_n, t_n) - f(\bm x_n, t_n)$, where $\bm x'_n$ is the vertex of $\mathcal{S}$ farthest from $\bm x_n$. (ii)~By Assumption~\ref{ass:covariance}, the observations collected before $t_n$ are, as far as this gap is concerned, a randomized function of noisy observations of a single stationary sequence with covariance $\left(k_T(h\Delta)\right)_{h \in \mathbb{Z}}$. By Lemma~\ref{lem:szego}, the posterior variance of the gap is therefore at least $\sigma^2_\star$. (iii)~The posterior mean of the gap is bounded from below with probability at least $1/2$, and a Gaussian with non-vanishing variance and bounded mean has a positive expected positive part. Because the bound on $\mathbb{P}[R_n \geq cn/2]$ does not vanish, every algorithm incurs a linear regret, even in probability. Since a nonzero continuous part in the spectral density $S_T$ is preserved by products with any other temporal kernel and by mixtures, we obtain the following result.

\begin{corollary} \label{cor:lower_classes}
    $\mathbb{E}[R_n] \in \Theta(n)$ for every TVBO algorithm and every $\Delta > 0$ whenever $k_T$ is a broadband kernel, a band-limited kernel, the product of one of them with any other temporal kernel satisfying Assumption~\ref{ass:covariance}, or a mixture giving a positive weight to one of them.
\end{corollary}

\paragraph{Relation to Existing Lower Bounds.} The temporal kernel $k_T(t_i - t_j) = (1-\varepsilon)^{|i-j|/2}$ of~\citet{bogunovic2016time}, where $\varepsilon \in (0, 1)$ is a forgetting factor, is an exponential kernel sampled at $t_i = i$. Its spectral density has no discrete part, so Theorem~\ref{thm:lower_regret_bound} recovers the linear asymptotics of their lower bound. Unlike the minimax result of~\citet{besbes2015non}, which exhibits a single hard objective, Theorem~\ref{thm:lower_regret_bound} holds in expectation, and with a probability bounded away from zero, over any GP prior whose temporal kernel has a spectral density with a continuous part.

\paragraph{Noise, Sampling Frequency and Nyquist Rate.} Refining the sampling grid (i.e., replacing $\Delta$ by $\Delta/k$, $k \in \mathbb{N}^*$) can only decrease $e_\infty$, since the past observations on the grid $\Delta\mathbb{Z}$ are also on the finer grid $(\Delta/k)\mathbb{Z}$, but it never cancels $e_\infty$ (see Figure~\ref{fig:prediction_rate} in Appendix~\ref{app:additional_figures} for empirical confirmation). The observational noise is essential. Without noise, Step~(ii) only yields the lower bound $2(1 - k_S(\bm x'_n, \bm x_n)) \exp\left(\frac{1}{2\pi}\int_{-\pi}^{\pi} \log s_\Delta(\theta) d\theta\right)$ on the posterior variance (Remark~\ref{rem:noiseless}), which vanishes when $s_\Delta$ vanishes. This has concrete practical consequences, since this is the case for objectives described by band-limited kernels ($\supp(S_T) = [-\tau, \tau]$) sampled above the Nyquist rate $1/\Delta > 2\tau$. For such objectives, the noiseless prediction error $v_m = \Var[X_0 | X_{-1}, \cdots, X_{-m}]$ is positive for every $m$ but vanishes as $m \to \infty$: they become perfectly predictable from a long enough noiseless past (Figure~\ref{fig:prediction_convergence} in Appendix~\ref{app:additional_figures}).

\paragraph{On the Constant.} The constant $c$ is positive but typically very small, because $\Psi(z)$ decays like $\phi(z)/z^2$ as $z \to -\infty$. Therefore, Theorem~\ref{thm:lower_regret_bound} does not forbid TVBO algorithms to be able to track an $\epsilon$-optimum as time goes by, which is a guarantee that is left for future work.

Theorem~\ref{thm:lower_regret_bound} also explains why all information-gain-based analyses of TVBO algorithms~\citep{bogunovic2016time, brunzemaevent} yield linear bounds.

\begin{proposition}[Information-gain barrier] \label{prop:barrier}
    Under Assumptions~\ref{ass:surrogate}-\ref{ass:sampling_freq}, $\gamma_n(k) \geq \rho n$ for all $n \in \mathbb{N}^*$. In particular, $\gamma_n(k) \in \Theta(n)$ whenever $S_T^c \neq 0$.
\end{proposition}

Proposition~\ref{prop:barrier} is proven in Appendix~\ref{app:barrier}: querying the same $\bm x$ at every step already collects $\rho$ nats of information per step, by Lemma~\ref{lem:szego}. Consequently, every regret bound in $\mathcal{O}(\sqrt{n\beta_n\gamma_n(k)})$ is at least linear when $S_T^c \neq 0$.

\section{Rhythmic Objectives Can Be Tracked} \label{sec:upper_bound}

Theorem~\ref{thm:lower_regret_bound} shows that the no-regret property requires $S_T^c = 0$. We now show that this condition, i.e., a rhythmic $k_T$ (Definition~\ref{def:rhythmic}), suffices for GP-UCB, under mild assumptions. Using $\cos(2\pi\omega(t - t')) = \cos(2\pi\omega t)\cos(2\pi\omega t') + \sin(2\pi\omega t)\sin(2\pi\omega t')$, a rhythmic kernel~\eqref{eq:rhythmic} can be written as $k_T(t - t') = \sum_{m \geq 0} w_m \psi_m(t)\psi_m(t')$, where $\psi_0 \equiv 1$, $w_0 = \alpha_0$, $\left(\psi_{2p-1}, \psi_{2p}\right) = \left(\cos(2\pi\omega_p \cdot), \sin(2\pi\omega_p \cdot)\right)$ for all $p \geq 1$ and $w_{2p-1} = w_{2p} = 2\alpha_p$. Hence, under Assumption~\ref{ass:covariance},
\begin{equation} \label{eq:moving_mixture}
    f(\bm x, t) \overset{d}{=} \sum_{m \geq 0} \sqrt{w_m}\, \psi_m(t)\, g_m(\bm x), \quad g_m \overset{\text{i.i.d.}}{\sim} \mathcal{GP}(0, k_S).
\end{equation}
Here, $\overset{d}{=}$ denotes equality in distribution. In other words, a rhythmic objective is a fixed collection of static functions whose mixing weights oscillate deterministically, so that all its randomness is drawn at $t = 0$.

With a low-rank $k_T$, TVBO reduces further to a contextual GP bandit with known contexts $\left(\psi_m(t_i)\right)_m$, for which \citet{krause2011contextual} already give a sublinear regret bound. The following lemma, proven in Appendix~\ref{app:info_gain_atomic}, controls the information gain of an arbitrary rhythmic kernel. Assume w.l.o.g.\ that the weights of the Dirac deltas are sorted, i.e., $\alpha_1 \geq \alpha_2 \geq \cdots$, and let $\varepsilon_P$ be the total weight of the Dirac deltas of $S_T$ beyond its $P$ heaviest pairs $\pm\omega_p$:
\begin{equation} \label{eq:tail_mass}
    \varepsilon_P = 2\sum_{p > P} \alpha_p, \quad P \in \mathbb{N}.
\end{equation}

\begin{lemma} \label{lem:info_gain_atomic}
    If $k_T$ is rhythmic, then for all $n, P \in \mathbb{N}$,
    \begin{equation} \label{eq:info_gain_atomic}
        \gamma_n(k) \leq (2P+1)\gamma_n(k_S) + \frac{n\varepsilon_P}{2\sigma^2_0}.
    \end{equation}
    In particular, $\gamma_n(k) \leq (2L+1)\gamma_n(k_S)$ if $k_T$ is a sum of $L$ cosines, i.e., if $\alpha_p = 0$ for all $p > L$.
\end{lemma}

Combined with a time-varying version of the analysis of~\citet{gpucb}, similar to the analysis of TV-GP-UCB by~\citet{bogunovic2016time} (Lemma~\ref{lem:gpucb_tv} in Appendix~\ref{app:upper_bound}), Lemma~\ref{lem:info_gain_atomic} yields, to the best of our knowledge, the first sublinear regret bound for TVBO.

\begin{theorem} \label{thm:upper_regret_bound}
    Let Assumptions~\ref{ass:surrogate}-\ref{ass:lipschitz} hold with a rhythmic $k_T$, let $\delta \in (0, 1)$ and let GP-UCB use
    \begin{equation} \label{eq:def_beta}
        \beta_i = 2\log\left(\frac{\pi^2 i^2}{2\delta}\right) + 2d\log\ceil*{i^2 d b\sqrt{\log\left(\frac{\pi^2 d a i^2}{2\delta}\right)}}.
    \end{equation}
    where $a$ and $b$ are given by Assumption~\ref{ass:lipschitz} and $\ceil{\cdot}$ denotes the ceiling function. Then, with probability at least $1 - \delta$, for all $n \in \mathbb{N}^*$,
    \begin{equation} \label{eq:upper_bound}
        R_n \leq \sqrt{C_1 n \beta_n \Gamma_n} + \frac{\pi^2}{6},
    \end{equation}
    where $C_1 = 8/\log(1 + \sigma_0^{-2})$ and $\Gamma_n = \min_{P \in \mathbb{N}} \left\{(2P+1)\gamma_n(k_S) + n\varepsilon_P/(2\sigma^2_0)\right\}$. In particular, $R_n \in o(n)$ with probability at least $1 - \delta$ whenever $\gamma_n(k_S) \in \mathcal{O}(n^{\eta})$ for some $\eta < 1$ and $\varepsilon_P \in o(1/\log P)$.
\end{theorem}

The condition on $\gamma_n(k_S)$ holds for the RBF kernel, for which $\gamma_n(k_S) \in \mathcal{O}((\log n)^{d+1})$~\citep{gpucb}, and for Matérn kernels with smoothness $\nu > 2$, for which $\gamma_n(k_S) \in \tilde{\mathcal{O}}(n^{d/(2\nu + d)})$~\citep{vakili2021information}. The condition on $\varepsilon_P$ only requires the weights of the Dirac deltas of $S_T$ to be summable at a slightly faster than logarithmic rate. Explicit rates follow from optimizing over $P$ (Appendix~\ref{app:rates}).

We now apply this result to the most popular spatial kernel (RBF) and to several interesting rhythmic temporal kernels, including the periodic kernel of~\citet{mackay1998introduction} and the low-rank kernels (i.e., sums of cosines).

\begin{corollary} \label{cor:rates}
    Under the mild condition on $\varepsilon_P$ in Theorem~\ref{thm:upper_regret_bound}, with an RBF spatial kernel $k_S$ and with probability at least $1 - \delta$,
    \begin{enumerate}[label=(\roman*), leftmargin=*, itemsep=0pt, topsep=2pt]
        \item $R_n \in \mathcal{O}\left(\sqrt{(2L+1) n (\log n)^{d+2}}\right)$ if $k_T$ is a sum of $L$ cosines,
        \item $R_n \in \mathcal{O}\left(\sqrt{\zeta^{-1} n (\log n)^{d+3}}\right)$ if $\varepsilon_P \leq C e^{-\zeta P}$ for some $C, \zeta > 0$. This holds for every $\zeta > 0$ when $k_T$ is the periodic kernel $k_T(t) = \exp\left(-2\sin^2(\pi t/r)/\ell^2\right)$, with period $r > 0$ and lengthscale $\ell > 0$,
        \item $R_n \in \tilde{\mathcal{O}}\left(n^{(2+\zeta)/(2+2\zeta)}\right)$ if $\varepsilon_P \leq C P^{-\zeta}$ for some $C, \zeta > 0$.
    \end{enumerate}
\end{corollary}

A static objective corresponds to $k_T \equiv 1$, i.e., to a sum of $L = 0$ cosines with $\alpha_0 = 1$. In that case, (i)~reads $R_n \in \mathcal{O}\left(\sqrt{n(\log n)^{d+2}}\right)$, which is the rate of GP-UCB with an RBF kernel~\citep{gpucb}. Together, Theorems~\ref{thm:lower_regret_bound} and~\ref{thm:upper_regret_bound} characterize the asymptotic performance of TVBO for all temporal kernels satisfying Assumption~\ref{ass:covariance}, up to the mild conditions of Theorem~\ref{thm:upper_regret_bound} on $\gamma_n(k_S)$ and $\varepsilon_P$.

\section{Experiments} \label{sec:experiments}

\paragraph{Setup.} We sample objectives exactly from the prior distribution $\mathcal{GP}(0, k_S k_T)$ on a regular grid of 100 points of $\mathcal{S} = [0, 1]$, with an RBF spatial kernel of lengthscale $0.1$, $\sigma^2_0 = 0.01$, $\Delta = 1$ and $n = 800$. We consider nine temporal kernels. Seven of them have a spectral density with a continuous part: RBF, Matérn-3/2 and exponential kernels (lengthscale $10$), sinc$^2$ kernels sampled above ($\tau = 0.05$) and below ($\tau = 0.7$) their Nyquist rate, $0.5$ periodic $+\ 0.5$ RBF, and periodic $\times$ RBF. The other two are rhythmic: a periodic kernel ($r = 40$, $\ell = 1$) and a sum of three cosines and a constant. GP-UCB (with a constant exploration parameter $\beta_i = 2$ to prevent over-exploration, see Figure~\ref{fig:theorem_d} for that of Theorem~\ref{thm:upper_regret_bound}) and expected improvement~(EI) use the true prior (no misspecified surrogate model). For the sake of completeness, we also run random search (uniform random queries in $\mathcal{S}$). We report the late-time regret, computed as the mean of $r_i$ over the last $25\%$ of the steps, over 20 objectives per kernel. Appendix~\ref{app:experiments} gives all details and Table~\ref{tab:results} all values. All experiments run in about 10 minutes on an 8-core laptop.

\begin{figure}[t]
  \centering
  \includegraphics[width=\columnwidth]{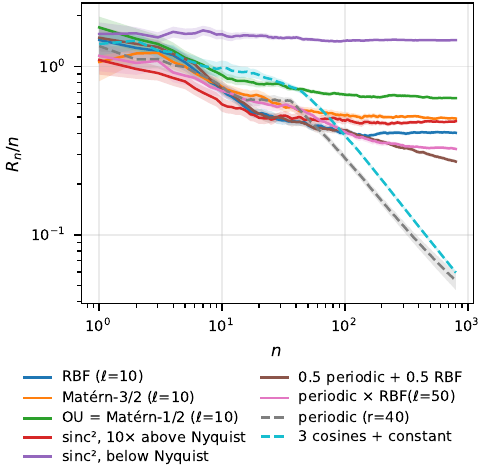}
  \caption{Average regret $R_n/n$ of GP-UCB on objectives sampled from the prior (20 objectives per kernel, $\pm$ one standard error). Solid lines: temporal kernels whose spectral density has a continuous part; dashed lines: rhythmic kernels.}
  \label{fig:regret}
\end{figure}

\paragraph{Two Regimes.} Figure~\ref{fig:regret} confirms Theorems~\ref{thm:lower_regret_bound} and~\ref{thm:upper_regret_bound}. For rhythmic kernels, the average regret of GP-UCB keeps decreasing whereas, for all seven kernels whose spectral density has a continuous part, it plateaus. This is not only the case for this acquisition function: the late-time regret of EI is indistinguishable from that of GP-UCB (Table~\ref{tab:results}). Band-limited kernels (e.g., sinc$^2$ sampled $10\times$ above its Nyquist rate) and mixture kernels (e.g., $0.5$ periodic $+\ 0.5$ RBF, half of whose spectral mass is concentrated on the atoms of the periodic kernel) also cause the average regret to plateau. The height of the latter plateau is governed by the weight of the broadband component, which also sets the information rate $\rho$ of the mixture. It remains positive, although smaller, for weights down to $0.01$ (Figure~\ref{fig:mixture} in Appendix~\ref{app:additional_figures}).

\begin{figure*}[t]
  \centering
  \begin{subfigure}[b]{0.6\textwidth}
    \includegraphics[width=\textwidth]{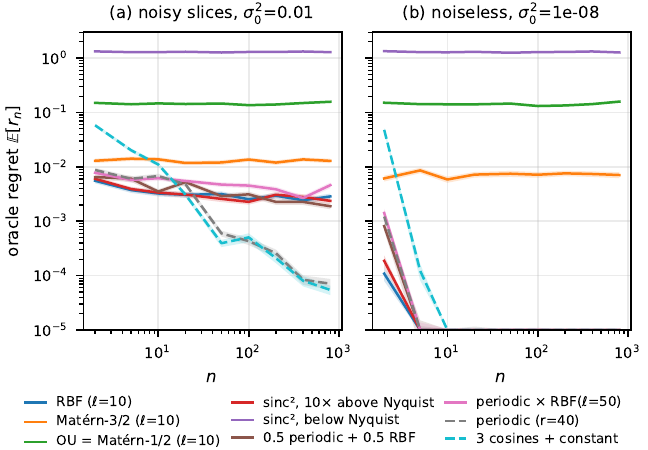}
  \end{subfigure}
  \quad
  \begin{subfigure}[b]{0.31\textwidth}
    \includegraphics[width=\textwidth]{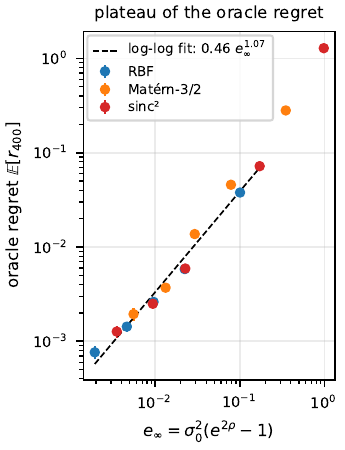}
  \end{subfigure}
  \caption{Instantaneous regret of the full-slice oracle, which lower-bounds that of every algorithm querying the grid (Proposition~\ref{prop:oracle}; 1000 Monte Carlo samples, $\pm$ one standard error). (Left)~$\mathbb{E}[r_n]$ as a function of $n$, with noisy slices ($\sigma^2_0 = 0.01$) and essentially noiseless slices ($\sigma^2_0 = 10^{-8}$). (Right)~Plateau $\mathbb{E}[r_{400}]$ as a function of $e_\infty = \sigma^2_0(e^{2\rho}-1)$ for 15 temporal kernels (RBF and Matérn-3/2 with different lengthscales and sinc$^2$ with different cut-off frequencies), with a log-log fit on $e_\infty$.}
  \label{fig:oracle}
\end{figure*}

\paragraph{An Algorithm-Free Test.} To test Theorem~\ref{thm:lower_regret_bound} independently of any acquisition function, we consider a setting similar to Appendix~F of~\citet{bogunovic2016time} and evaluate the performance of an oracle that observes the whole noisy slices $f(\cdot, t_j) + \epsilon(\cdot, t_j)$ on the grid for all $j < n$, where the noises $\epsilon(\bm x, t_j) \sim \mathcal{N}(0, \sigma^2_0)$ are i.i.d.\ across grid points and times, and selects the maximizer $\argmax_{\bm x \in \mathcal{S}} \mathbb{E}\left[f(\bm x, t_n)\mid \mathcal{D}_n\right]$ of its posterior mean at time $t_n$. Its regret lower-bounds that of every algorithm querying the grid (Proposition~\ref{prop:oracle} in Appendix~\ref{app:oracle}). With noise, the regret of the oracle plateaus for all seven kernels with a continuous spectral part (Figure~\ref{fig:oracle}, left), including sinc$^2$ above its Nyquist rate, whereas it keeps decreasing for rhythmic kernels. Without noise, the oversampled sinc$^2$ and RBF kernels behave like rhythmic ones, since their asymptotic prediction errors are tiny ($e_\infty \approx 7 \cdot 10^{-8}$ and $2 \cdot 10^{-7}$ for $\sigma^2_0 = 10^{-8}$), so that this experiment cannot distinguish them.

\paragraph{The Price of Tracking.} The plateaus of the oracle regret for 15 temporal kernels from three families are approximately proportional to $e_\infty$ over two orders of magnitude (Figure~\ref{fig:oracle}, right). GP-UCB is much farther from this limit. Its plateaus are up to $200$ times higher than those of the oracle for slowly varying kernels, and grow more slowly with $e_\infty$ (log-log slope $0.46$ instead of $1.07$), but they are still ordered by $\rho$, with a Spearman rank correlation of $0.95$ (Figure~\ref{fig:difficulty} in Appendix~\ref{app:additional_figures}).

\paragraph{Barrier and Prediction Theory.} Appendix~\ref{app:additional_figures} confirms the remaining claims. The information gain of the fixed-location design converges to $\rho n$, whereas that of a greedy design stays close to $2$ nats per step for all kernels with a continuous spectral part (Figure~\ref{fig:info_gain}). The exact finite-$n$ prediction errors converge to the formulas of Lemma~\ref{lem:szego} (Figure~\ref{fig:prediction_convergence}).

\section{Conclusion} \label{sec:conclusion}

This paper studies the asymptotic performance of TVBO under the assumptions of the literature (Section~\ref{sec:background-assumptions}), for all four classes of stationary temporal kernels (broadband, band-limited, almost-periodic and low-rank) and their mixtures. If the spectral density of $k_T$ has a continuous part, as for broadband and band-limited kernels, the sampled temporal process keeps producing information at a positive rate $\rho$. Then, no algorithm is no-regret (Theorem~\ref{thm:lower_regret_bound}) and information-gain analyses are bound to be linear (Proposition~\ref{prop:barrier}). If it is a mixture of Dirac deltas, as for almost-periodic and low-rank kernels, GP-UCB is no-regret under mild conditions (Theorem~\ref{thm:upper_regret_bound}). Table~\ref{tab:wrapup} summarizes these insights. In short, what makes a time-varying objective trackable is not how slowly it moves, but whether the algorithm keeps learning something new at each time step.

Since real-world dynamics are rarely exactly almost-periodic or low-rank, these results suggest that no-regret is the wrong target for practical TVBO. Algorithms should rather be compared through their regret slope, whose information-theoretic limit appears to be approximately proportional to $e_\infty$ (Section~\ref{sec:experiments}).

This work also opens up new research questions. Can this approximate proportionality be made precise, which would make $\rho$ a quantitative measure of difficulty? How can the large gap between point queries and the full-slice oracle be closed? What happens for non-separable covariance functions, or when Assumption~\ref{ass:sampling_freq} is relaxed? The latter is crucial in practice, since a new observation is often sampled only after performing GP inference, so that observations are not collected at a fixed frequency~\citep{bardou2024too}. The tools of this paper, i.e., prediction theory applied to the sampled temporal process, will likely help answer these questions.

\bibliography{main}

\clearpage
\appendix
\onecolumn

\section{Spectral Background} \label{app:spectral}

This appendix gathers the spectral tools used in our proofs. Appendix~\ref{app:sampled} describes how sampling the temporal process at times $t_i = i\Delta$ transforms its spectral density (Lemma~\ref{lem:sampled}), Appendix~\ref{app:szego} proves Lemma~\ref{lem:szego} and Appendix~\ref{app:examples} lists the spectral densities of common temporal kernels. We start by recalling the relevant notations and by introducing the spectral density of a sequence of random variables.

\paragraph{Spectral Density of the Temporal Kernel.} By Assumption~\ref{ass:covariance} and Bochner's theorem~\citep{bochner2005harmonic}, the temporal kernel $k_T$ is the Fourier transform of its spectral density $S_T$, i.e.,
\begin{equation} \label{eq:app_bochner}
    k_T(t) = \int_{\mathbb{R}} S_T(\omega) e^{2\pi \iu \omega t} d\omega, \quad t \in \mathbb{R}.
\end{equation}

The spectral density $S_T$ is the sum of a continuous part $S_T^c$ and of a discrete part made of Dirac deltas, that is,
\begin{equation} \label{eq:app_spectral_form}
    S_T(\omega) = S_T^c(\omega) + \alpha_0\delta(\omega) + \sum_{p \geq 1} \alpha_p\left(\delta(\omega - \omega_p) + \delta(\omega + \omega_p)\right),
\end{equation}
where $S_T^c : \mathbb{R} \to \mathbb{R}_{\geq 0}$ is even and integrable, $\alpha_0, \alpha_1, \cdots \geq 0$ are the weights of the Dirac deltas and $\omega_1, \omega_2, \cdots > 0$ are distinct frequencies. Since each Dirac delta integrates to $1$, taking $t = 0$ in~\eqref{eq:app_bochner} and using $k_T(0) = 1$ gives
\begin{equation} \label{eq:app_total_mass}
    \int_{\mathbb{R}} S_T^c(\omega) d\omega + \alpha_0 + 2\sum_{p \geq 1} \alpha_p = 1.
\end{equation}
In particular, the series $\sum_{p \geq 1} \alpha_p$ converges.

\paragraph{Spectral Density of a Stationary Sequence.} A sequence of random variables $(Z_i)_{i \in \mathbb{Z}}$ is centered and stationary if $\mathbb{E}[Z_i] = 0$ for all $i \in \mathbb{Z}$ and if its autocovariance $\Cov(Z_i, Z_{i+h})$ only depends on the lag $h \in \mathbb{Z}$. Its spectral density is a nonnegative function $f_Z$ on $(-\pi, \pi]$, which may contain Dirac deltas, such that
\begin{equation} \label{eq:app_sequence_density}
    \Cov(Z_i, Z_{i+h}) = \frac{1}{2\pi}\int_{(-\pi, \pi]} f_Z(\theta) e^{\iu h\theta} d\theta \quad \text{for all } i, h \in \mathbb{Z}.
\end{equation}
For every integer $h$, $\theta \mapsto e^{\iu h\theta}$ is $2\pi$-periodic, so that angular frequencies are considered modulo $2\pi$. This is why $f_Z$ is defined on an interval of length $2\pi$, which we take half-open so that a Dirac delta located at $\pi$ is counted exactly once. By Herglotz's theorem~\citep[Section~4.3]{brockwell1991time}, the autocovariances $\left(\Cov(Z_0, Z_h)\right)_{h \in \mathbb{Z}}$ determine the spectral density uniquely. Hence, to identify the spectral density of a stationary sequence, it suffices to exhibit a function $f_Z$ that satisfies~\eqref{eq:app_sequence_density}. Two simple consequences of~\eqref{eq:app_sequence_density} will be used many times:
\begin{itemize}
    \item a Dirac delta $2\pi\alpha\,\delta(\theta - \theta')$ in $f_Z$, with $\alpha \geq 0$ and $\theta' \in (-\pi, \pi]$, contributes $\alpha e^{\iu h\theta'}$ to $\Cov(Z_i, Z_{i+h})$;
    \item a white noise of variance $\sigma^2_0$ has the constant spectral density $f_Z = \sigma^2_0$, since $\frac{1}{2\pi}\int_{(-\pi, \pi]} \sigma^2_0 e^{\iu h\theta} d\theta$ is $\sigma^2_0$ if $h = 0$ and $0$ otherwise.
\end{itemize}

\subsection{Spectral Density of the Sampled Temporal Process} \label{app:sampled}

Lemma~\ref{lem:sampled} (and Lemma~\ref{lem:szego}) only involves the temporal kernel $k_T$: the spatial component does not play any role in it. It is nevertheless relevant for TVBO because of the separable structure~\eqref{eq:decomposable_covariance} of $k$. For any fixed $\bm x \in \mathcal{S}$, the slice $g = f(\bm x, \cdot)$ is a $\mathcal{GP}(0, k_T)$, since $\Cov\left(f(\bm x, t), f(\bm x, t')\right) = k_S(\bm x, \bm x)k_T(t - t') = k_T(t - t')$. Since observations are collected at times $t_i = i\Delta$ (Assumption~\ref{ass:sampling_freq}), an algorithm that repeatedly queries $\bm x$ observes $g$ through the sequence of samples $\left(g(i\Delta)\right)_{i \in \mathbb{Z}}$, whose autocovariances are the samples $\left(k_T(h\Delta)\right)_{h \in \mathbb{Z}}$ of $k_T$. Our proofs reduce the spatio-temporal problem to such a temporal sequence. The proof of Proposition~\ref{prop:barrier} uses exactly this fixed-location design (Appendix~\ref{app:barrier}), and the proof of Theorem~\ref{thm:lower_regret_bound} applies the same argument to the normalized difference between the values of $f$ at two locations (Appendix~\ref{app:lower_bound}, Step~2). In this appendix, we therefore only consider a temporal process $g \sim \mathcal{GP}(0, k_T)$ and its samples $\left(g(i\Delta)\right)_{i \in \mathbb{Z}}$. Lemma~\ref{lem:sampled} computes the spectral density of this sequence from $S_T$. The key phenomenon is aliasing: at the sampling times, the frequencies $\omega$ and $\omega + \ell/\Delta$, with $\ell \in \mathbb{Z}$, produce the same complex exponential, since, for all $h \in \mathbb{Z}$,
\begin{equation*}
    e^{2\pi \iu (\omega + \ell/\Delta)h\Delta} = e^{2\pi \iu \omega h\Delta} e^{2\pi \iu \ell h} = e^{2\pi \iu \omega h\Delta}.
\end{equation*}
Frequencies that differ by a multiple of the sampling frequency $1/\Delta$ are therefore indistinguishable from the samples, and their spectral masses add up. To formalize this, we introduce the folding map $\pi_\Delta : \mathbb{R} \to (-\pi, \pi]$, which maps a frequency $\omega$ to the unique representative $\pi_\Delta(\omega) \in (-\pi, \pi]$ of $2\pi\Delta\omega$ modulo $2\pi$. In particular, $\pi_\Delta(0) = 0$ and $\pi_\Delta(\omega + \ell/\Delta) = \pi_\Delta(\omega)$ for all $\ell \in \mathbb{Z}$.

\begin{lemma} \label{lem:sampled}
    Let Assumption~\ref{ass:covariance} hold, so that $S_T$ is of the form~\eqref{eq:app_spectral_form}. Let $\Delta > 0$, $g \sim \mathcal{GP}(0, k_T)$ and $X_i = g(i\Delta)$ for all $i \in \mathbb{Z}$. For all $\theta \in (-\pi, \pi]$, let
    \begin{equation} \label{eq:app_sampled_density}
        s_\Delta(\theta) = \frac{1}{\Delta}\sum_{\ell \in \mathbb{Z}} S_T^c\left(\frac{\theta + 2\pi\ell}{2\pi\Delta}\right) \in [0, +\infty].
    \end{equation}
    Then,
    \begin{enumerate}[label=(\roman*)]
        \item $s_\Delta$ is integrable on $(-\pi, \pi]$, with $\int_{(-\pi, \pi]} s_\Delta(\theta) d\theta = 2\pi\int_{\mathbb{R}} S_T^c(\omega) d\omega$. In particular, $s_\Delta$ is finite almost everywhere;
        \item $(X_i)_{i \in \mathbb{Z}}$ is a centered stationary Gaussian sequence whose spectral density is
        \begin{equation} \label{eq:app_sampled_full_density}
            f_X(\theta) = s_\Delta(\theta) + 2\pi\alpha_0\delta(\theta) + 2\pi\sum_{p \geq 1}\alpha_p\left(\delta(\theta - \pi_\Delta(\omega_p)) + \delta(\theta - \pi_\Delta(-\omega_p))\right);
        \end{equation}
        \item $s_\Delta$ is positive on a set of positive measure if and only if $S_T^c \neq 0$, i.e., if and only if $\int_{\mathbb{R}} S_T^c(\omega) d\omega > 0$.
    \end{enumerate}
\end{lemma}

The function $s_\Delta$ is the continuous part of $f_X$ (introduced in~\eqref{eq:sampled_density}). For every $\theta \in (-\pi, \pi]$, the frequencies $\omega$ that fold onto $\theta$, i.e., such that $\pi_\Delta(\omega) = \theta$, are exactly in $\left\{(\theta + 2\pi\ell)/(2\pi\Delta)\right\}_{\ell \in \mathbb{Z}}$. Hence, $s_\Delta(\theta)$ sums the values of $S_T^c$ at all the frequencies that fold onto $\theta$. Similarly, each Dirac delta of $S_T$ at a frequency $\omega'$ becomes a Dirac delta of $f_X$ at the folded frequency $\pi_\Delta(\omega')$. The factors $1/\Delta$ and $2\pi$ in~\eqref{eq:app_sampled_density} and~\eqref{eq:app_sampled_full_density} come from the change of variable $\theta = 2\pi\Delta\omega - 2\pi\ell$, whose Jacobian is $d\omega/d\theta = 1/(2\pi\Delta)$, and from the factor $1/(2\pi)$ in the normalization~\eqref{eq:app_sequence_density}.

\begin{proof}
    We first compute the autocovariances of $(X_i)_{i \in \mathbb{Z}}$ and split them into the contributions of the continuous and discrete parts of $S_T$ (Step~1). We then write each contribution in the form~\eqref{eq:app_sequence_density} (Steps~2 and~3), which proves~(i) along the way, and we conclude (Step~4).

    \emph{Step 1: Autocovariances of $(X_i)$.} Since $g$ is a centered GP, the random variables $X_i$ are jointly Gaussian and centered. For all $i, h \in \mathbb{Z}$, we have $\Cov(X_i, X_{i+h}) = k_T((i + h)\Delta - i\Delta) = k_T(h\Delta)$, which does not depend on $i$. Hence, $(X_i)_{i \in \mathbb{Z}}$ is a centered stationary Gaussian sequence. Now, fix $h \in \mathbb{Z}$. Plugging~\eqref{eq:app_spectral_form} of $S_T$ into~\eqref{eq:app_bochner} with $t = h\Delta$, and using $\int_{\mathbb{R}} \delta(\omega - \omega')e^{2\pi \iu \omega h\Delta} d\omega = e^{2\pi \iu \omega'h\Delta}$ for every Dirac delta, we obtain
    \begin{equation} \label{eq:app_cov_split}
        k_T(h\Delta) = C_h + D_h,
    \end{equation}
    where
    \begin{equation*}
        C_h = \int_{\mathbb{R}} S_T^c(\omega) e^{2\pi \iu \omega h\Delta} d\omega \quad \text{and} \quad D_h = \alpha_0 + \sum_{p \geq 1} \alpha_p\left(e^{2\pi \iu \omega_ph\Delta} + e^{-2\pi \iu \omega_ph\Delta}\right)
    \end{equation*}
    are the contributions of the continuous part and of the discrete part of $S_T$ to $k_T(h\Delta)$, respectively. It remains to write $C_h$ and $D_h$ as integrals over $(-\pi, \pi]$ of the form~\eqref{eq:app_sequence_density}.

    \emph{Step 2: Discrete part.} Let $\omega' \in \mathbb{R}$. By definition of $\pi_\Delta$, there exists $\ell' \in \mathbb{Z}$ such that $\pi_\Delta(\omega') = 2\pi\Delta\omega' - 2\pi\ell'$. Since $h\ell' \in \mathbb{Z}$, we have $e^{\iu h\pi_\Delta(\omega')} = e^{2\pi \iu \omega'h\Delta}e^{-2\pi \iu h\ell'} = e^{2\pi \iu \omega'h\Delta}$. Therefore, for any weight $\alpha \geq 0$,
    \begin{equation} \label{eq:app_fold_dirac}
        \alpha e^{2\pi \iu \omega'h\Delta} = \alpha e^{\iu h\pi_\Delta(\omega')} = \frac{1}{2\pi}\int_{(-\pi, \pi]} 2\pi\alpha\,\delta(\theta - \pi_\Delta(\omega')) e^{\iu h\theta} d\theta.
    \end{equation}
    We apply~\eqref{eq:app_fold_dirac} to each term of $D_h$, i.e., to $(\alpha, \omega') = (\alpha_0, 0)$, for which $\pi_\Delta(0) = 0$, and to $(\alpha, \omega') = (\alpha_p, \omega_p)$ and $(\alpha_p, -\omega_p)$ for all $p \geq 1$. Summing the resulting identities gives
    \begin{equation} \label{eq:app_discrete_part}
        D_h = \frac{1}{2\pi}\int_{(-\pi, \pi]} \left(2\pi\alpha_0\delta(\theta) + 2\pi\sum_{p \geq 1}\alpha_p\left(\delta(\theta - \pi_\Delta(\omega_p)) + \delta(\theta - \pi_\Delta(-\omega_p))\right)\right) e^{\iu h\theta} d\theta,
    \end{equation}
    where the infinite sum and the integral can be exchanged because $\sum_{p \geq 1} \alpha_p$ converges and $\left|e^{\iu h\theta}\right| = 1$.

    \emph{Step 3: Continuous part and proof of~(i).} The continuous part requires more care, because the mass of $S_T^c$ is spread over the whole real line instead of being concentrated at countably many frequencies. We cut the real line into intervals of width $1/\Delta$ on which $\pi_\Delta$ has a simple expression. For all $\ell \in \mathbb{Z}$, let $\mathcal{I}_\ell = \left((2\ell - 1)/(2\Delta), (2\ell + 1)/(2\Delta)\right]$. The intervals $(\mathcal{I}_\ell)_{\ell \in \mathbb{Z}}$ are disjoint and cover $\mathbb{R}$. For $\omega \in \mathcal{I}_\ell$, we have $2\pi\Delta\omega \in \left((2\ell - 1)\pi, (2\ell + 1)\pi\right]$, and thus $\pi_\Delta(\omega) = 2\pi\Delta\omega - 2\pi\ell$. Consequently, the change of variable $\theta = 2\pi\Delta\omega - 2\pi\ell$ maps $\mathcal{I}_\ell$ onto $(-\pi, \pi]$ and transforms $e^{2\pi \iu \omega h\Delta}$ into $e^{\iu h\theta}e^{2\pi \iu \ell h} = e^{\iu h\theta}$. For all $\ell \in \mathbb{Z}$, this gives
    \begin{equation} \label{eq:app_fold_interval}
        \int_{\mathcal{I}_\ell} S_T^c(\omega)e^{2\pi \iu \omega h\Delta} d\omega = \frac{1}{2\pi\Delta}\int_{(-\pi, \pi]} S_T^c\left(\frac{\theta + 2\pi\ell}{2\pi\Delta}\right)e^{\iu h\theta} d\theta.
    \end{equation}
    We first take $h = 0$ to prove~(i). In that case, both sides of~\eqref{eq:app_fold_interval} are nonnegative. Summing~\eqref{eq:app_fold_interval} over $\ell \in \mathbb{Z}$ and exchanging the sum and the integral on the right-hand side, we get
    \begin{equation*}
        \int_{\mathbb{R}} S_T^c(\omega) d\omega = \sum_{\ell \in \mathbb{Z}}\int_{\mathcal{I}_\ell} S_T^c(\omega) d\omega = \frac{1}{2\pi}\int_{(-\pi, \pi]} \frac{1}{\Delta}\sum_{\ell \in \mathbb{Z}} S_T^c\left(\frac{\theta + 2\pi\ell}{2\pi\Delta}\right) d\theta = \frac{1}{2\pi}\int_{(-\pi, \pi]} s_\Delta(\theta) d\theta,
    \end{equation*}
    where the first equality holds because the intervals $\mathcal{I}_\ell$ partition $\mathbb{R}$ and the last one is the definition~\eqref{eq:app_sampled_density} of $s_\Delta$. Since $S_T^c$ is integrable, $s_\Delta$ is integrable. Because $s_\Delta$ is a nonnegative function with a finite integral, it is finite almost everywhere, which proves~(i).

    We now consider an arbitrary $h \in \mathbb{Z}$. We have
    \begin{align}
        C_h &= \sum_{\ell \in \mathbb{Z}}\int_{\mathcal{I}_\ell} S_T^c(\omega)e^{2\pi \iu \omega h\Delta} d\omega \label{eq:app_cont_partition}\\
        &= \sum_{\ell \in \mathbb{Z}}\frac{1}{2\pi\Delta}\int_{(-\pi, \pi]} S_T^c\left(\frac{\theta + 2\pi\ell}{2\pi\Delta}\right)e^{\iu h\theta} d\theta \label{eq:app_cont_change}\\
        &= \frac{1}{2\pi}\int_{(-\pi, \pi]} \left(\frac{1}{\Delta}\sum_{\ell \in \mathbb{Z}} S_T^c\left(\frac{\theta + 2\pi\ell}{2\pi\Delta}\right)\right)e^{\iu h\theta} d\theta \label{eq:app_cont_fubini}\\
        &= \frac{1}{2\pi}\int_{(-\pi, \pi]} s_\Delta(\theta)e^{\iu h\theta} d\theta, \label{eq:app_cont_def}
    \end{align}
    where~\eqref{eq:app_cont_partition} holds because the intervals $\mathcal{I}_\ell$ partition $\mathbb{R}$ and $S_T^c(\omega)e^{2\pi \iu \omega h\Delta}$ is integrable,~\eqref{eq:app_cont_change} uses~\eqref{eq:app_fold_interval},~\eqref{eq:app_cont_fubini} exchanges the sum and the integral and~\eqref{eq:app_cont_def} uses the definition~\eqref{eq:app_sampled_density} of $s_\Delta$.

    \emph{Step 4: Proof of~(ii) and~(iii).} Plugging~\eqref{eq:app_discrete_part} and~\eqref{eq:app_cont_def} into~\eqref{eq:app_cov_split}, we obtain, for all $i, h \in \mathbb{Z}$,
    \begin{equation*}
        \Cov(X_i, X_{i+h}) = k_T(h\Delta) = C_h + D_h = \frac{1}{2\pi}\int_{(-\pi, \pi]} f_X(\theta) e^{\iu h\theta} d\theta,
    \end{equation*}
    with $f_X$ given by~\eqref{eq:app_sampled_full_density}. Hence, $f_X$ satisfies~\eqref{eq:app_sequence_density} for $(X_i)_{i \in \mathbb{Z}}$. By the uniqueness of the spectral density given by Herglotz's theorem, $f_X$ is the spectral density of $(X_i)_{i \in \mathbb{Z}}$, which proves~(ii). Finally, the integral of a nonnegative function is positive if and only if the function is positive on a set of positive measure. Since $s_\Delta \geq 0$, (i) shows that $s_\Delta$ is positive on a set of positive measure if and only if $\int_{\mathbb{R}} S_T^c(\omega) d\omega > 0$, which proves~(iii).
\end{proof}

\subsection{Kolmogorov--Szeg\H{o} Formula and Proof of Lemma~\ref{lem:szego}} \label{app:szego}

Let us first recall the Kolmogorov--Szeg\H{o} formula, which is a classical result in prediction theory.

\begin{lemma}[Kolmogorov--Szeg\H{o} formula; \citealp{grenander1958toeplitz}, \citealp{brockwell1991time}, Theorem~5.8.1] \label{lem:kolmogorov}
    Let $(Z_i)_{i \in \mathbb{Z}}$ be a centered stationary Gaussian sequence whose spectral density is the sum of an integrable function $f_Z$ and of Dirac deltas. Then, $\Var\left[Z_0 | Z_{-1}, \cdots, Z_{-m}\right]$ is non-increasing in $m$ and converges to $\exp\left(\frac{1}{2\pi}\int_{-\pi}^{\pi}\log f_Z(\theta) d\theta\right)$ as $m \to \infty$, with the convention $\exp(-\infty) = 0$.
\end{lemma}

We are now ready to prove Lemma~\ref{lem:szego}, which is a direct application of Lemmas~\ref{lem:sampled} and~\ref{lem:kolmogorov}.

\begin{proof}[Proof of Lemma~\ref{lem:szego}]
    (i)~Since $\log(1 + x) > 0$ for $x > 0$ and $\log(1+x) \leq x$, we have $0 \leq \rho \leq \frac{1}{4\pi\sigma^2_0}\int s_\Delta < \infty$, and $\rho > 0$ if and only if $s_\Delta > 0$ on a set of positive measure, which, by Lemma~\ref{lem:sampled}, holds if and only if $S_T^c \neq 0$.

    (ii)~The sequence $(Y_i)_{i \in \mathbb{Z}}$ is centered, stationary and Gaussian, whose spectral density is that of $(X_i)$ plus $\sigma^2_0$ (Lemma~\ref{lem:sampled}). Its continuous part is thus $s_\Delta + \sigma^2_0$. Since $\epsilon_0$ is independent of $(X_0, Y_{-1}, \cdots, Y_{-m})$,
    \begin{equation*}
        \Var\left[Y_0 | Y_{-1}, \cdots, Y_{-m}\right] = \Var\left[X_0 | Y_{-1}, \cdots, Y_{-m}\right] + \sigma^2_0 = e_m + \sigma^2_0.
    \end{equation*}
    By Lemma~\ref{lem:kolmogorov}, $e_m + \sigma^2_0$ is non-increasing and converges to $\exp\left(\frac{1}{2\pi}\int\log\left(s_\Delta + \sigma^2_0\right)\right) = \sigma^2_0 \exp\left(\frac{1}{2\pi}\int\log\left(1 + s_\Delta/\sigma^2_0\right)\right) = \sigma^2_0 e^{2\rho}$.
\end{proof}

\begin{remark}[Noiseless Observations] \label{rem:noiseless}
    Applying Lemma~\ref{lem:kolmogorov} to $(X_i)$ directly shows that $v_m = \Var[X_0 | X_{-1}, \cdots, X_{-m}]$ decreases to $\exp\left(\frac{1}{2\pi}\int\log s_\Delta\right)$, which is positive if and only if $\int \log s_\Delta > -\infty$. If $S_T = S_T^c$ with $\supp(S_T) = [-\tau, \tau]$ and $1/\Delta > 2\tau$, then $s_\Delta = 0$ on $\left\{\theta : 2\pi\tau\Delta < |\theta| \leq \pi\right\}$, so that $v_m \to 0$. Perfect prediction is, however, only achieved asymptotically: $v_m > 0$ for every $m \in \mathbb{N}$ as soon as $S_T^c \neq 0$. Indeed, for any nonzero $\bm a \in \mathbb{R}^{m+1}$, $\Var\left[\sum_{j=0}^m a_jX_{-j}\right] \geq \frac{1}{2\pi}\int_{-\pi}^{\pi}\left|\sum_{j=0}^m a_je^{\iu j\theta}\right|^2 s_\Delta(\theta) d\theta > 0$ by Lemma~\ref{lem:sampled}, because a nonzero trigonometric polynomial only vanishes at finitely many points and $s_\Delta$ is positive on a set of positive measure. The covariance matrix of $(X_0, X_{-1}, \cdots, X_{-m})$ is thus positive definite, and $v_m$, which is the ratio of the determinants of this matrix and of the covariance matrix of $(X_{-1}, \cdots, X_{-m})$, is positive. If $S_T^c$ is continuous and positive on $(-\tau, \tau)$ and $1/\Delta < 2\tau$, then $s_\Delta(\theta) \geq S_T^c(\theta/(2\pi\Delta))/\Delta$ is bounded away from $0$ and $v_m \to v_\infty > 0$. The same argument applies, for every $\Delta > 0$, to broadband kernels whose $S_T^c$ is continuous and positive on $\mathbb{R}$ (e.g., RBF and Matérn kernels), although the limit can be numerically negligible.
\end{remark}

\subsection{Spectral Densities of Common Temporal Kernels} \label{app:examples}

Table~\ref{tab:spectra} gives the spectral densities of the temporal kernels used in this paper (frequencies in cycles per unit of time). The RBF, Matérn, rational quadratic~(a scale mixture of RBF kernels~\citep{williams2006gaussian}), exponential (Matérn-1/2), sinc and sinc$^2$ kernels have no discrete part: they are broadband (first four) or band-limited (last two). The periodic kernel satisfies $\exp\left(-2\sin^2(\pi t/r)/\ell^2\right) = e^{-z}\exp\left(z\cos(2\pi t/r)\right) = \sum_{p \in \mathbb{Z}} e^{-z}I_{|p|}(z) e^{2\pi \iu pt/r}$ with $z = \ell^{-2}$, where $I_p$ is the modified Bessel function of the first kind~\citep{olver2010nist}. It is therefore almost-periodic, with $\alpha_p = e^{-z}I_p(z)$ and $\omega_p = p/r$ in~\eqref{eq:rhythmic}.

The spectral density of a product $k_T = k_1k_2$ is the convolution $S_1 * S_2$. If $S_1$ and $S_2$ are of the form~\eqref{eq:app_spectral_form}, so is $S_1 * S_2$. Its discrete part is the convolution of the discrete parts, and its continuous part $S_1^c * S_2^c + S_1^c * (S_2 - S_2^c) + (S_1 - S_1^c) * S_2^c$ has integral at least $\int S_1^c \cdot \int S_2 = \int S_1^c$, since $\int S_2 = 1$. Hence, the continuous part of a product is nonzero as soon as that of one factor is. The spectral density of a mixture $\sum_j \lambda_j k_j$ is $\sum_j \lambda_j S_j$, whose continuous part is nonzero as soon as some $\lambda_j > 0$ multiplies a spectral density with a nonzero continuous part. Corollary~\ref{cor:lower_classes} follows from these facts and Theorem~\ref{thm:lower_regret_bound}.

\begin{table}[h]
    \caption{Spectral densities $S_T$ of common temporal kernels $k_T$, with lengthscale $\ell > 0$, Matérn smoothness $\nu > 0$, cutoff frequency $\tau > 0$ and period $r > 0$. $\Gamma$ is the Gamma function, $K_\nu$ is the modified Bessel function of the second kind, $I_p$ is the modified Bessel function of the first kind, $\indic\{\cdot\}$ is the indicator function and $(x)_+ = \max(x, 0)$.}
    \label{tab:spectra}
    \centering
    \begin{tabular}{lll}
        \toprule
        Kernel & $k_T(t)$ & $S_T(\omega)$ \\
        \midrule
        RBF & $\exp\left(-t^2/(2\ell^2)\right)$ & $\sqrt{2\pi}\ell\exp\left(-2\pi^2\ell^2\omega^2\right)$ \\
        Exponential & $\exp\left(-|t|/\ell\right)$ & $2\ell/\left(1 + 4\pi^2\ell^2\omega^2\right)$ \\
        Matérn-$\nu$ & $\frac{2^{1-\nu}}{\Gamma(\nu)}\left(\frac{\sqrt{2\nu}|t|}{\ell}\right)^\nu K_\nu\left(\frac{\sqrt{2\nu}|t|}{\ell}\right)$ & $\propto \left(2\nu/\ell^2 + 4\pi^2\omega^2\right)^{-(\nu + 1/2)}$ \\
        Sinc & $\sin(2\pi\tau t)/(2\pi\tau t)$ & $\indic\{|\omega| \leq \tau\}/(2\tau)$ \\
        Sinc$^2$ & $\left(\sin(\pi\tau t)/(\pi\tau t)\right)^2$ & $\left(1 - |\omega|/\tau\right)_+/\tau$ \\
        Periodic & $\exp\left(-2\sin^2(\pi t/r)/\ell^2\right)$ & $\sum_{p \in \mathbb{Z}} e^{-\ell^{-2}}I_{|p|}(\ell^{-2})\,\delta(\omega - p/r)$ \\
        Sum of cosines & $\alpha_0 + 2\sum_{p=1}^L\alpha_p\cos(2\pi\omega_pt)$ & $\alpha_0\delta(\omega) + \sum_{p=1}^L\alpha_p\left(\delta(\omega - \omega_p) + \delta(\omega + \omega_p)\right)$ \\
        \bottomrule
    \end{tabular}
\end{table}

\section{Proof of the Lower Regret Bound} \label{app:lower_bound}

\paragraph{Preliminaries.} Let $c_S = \min_{\norm{\bm u - \bm v}_2 \geq 1/2} \left(1 - k_S(\bm u, \bm v)\right)$ be the minimum of the continuous and positive function $1 - k_S$ on the compact set $\left\{(\bm u, \bm v) \in \mathcal{S}^2 : \norm{\bm u - \bm v}_2 \geq 1/2\right\}$, which is nonempty since $d \geq 1$. By Assumption~\ref{ass:covariance}, $c_S > 0$. By the mean value theorem, $\range(g_S) \leq \sum_{j=1}^d \sup_{\bm x}|\partial g_S(\bm x)/\partial x_j|$ on $\mathcal{S} = [0, 1]^d$. Assumption~\ref{ass:lipschitz} then gives, for $q \in \{1, 2\}$, $\mathbb{E}\left[\sup_{\bm x}|\partial g_S/\partial x_j|^q\right] = \int_0^\infty qu^{q-1}\mathbb{P}\left[\sup_{\bm x}|\partial g_S/\partial x_j| > u\right] du \leq \int_0^\infty qu^{q-1}\min\left(1, a e^{-(u/b)^2}\right) du < \infty$. Therefore, $\bar R = \mathbb{E}\left[\range(g_S)\right]$ and $\mathbb{E}\left[\range(g_S)^2\right]$ are both finite. Since $k_T(0) = 1$, every temporal slice of the objective $f(\cdot, t_i)$ has the same law as $g_S$. Finally, $e_\infty = \sigma^2_0\left(e^{2\rho} - 1\right) > 0$ by Lemma~\ref{lem:szego} and $e_\infty \leq e_0 = 1$.

Fix a TVBO algorithm and $n \in \mathbb{N}^*$. Let $\mathcal{G}_n$ be the $\sigma$-algebra generated by $\mathcal{D}_{n-1}$ and $U$, so that $\bm x_n$ is $\mathcal{G}_n$-measurable. Since $U$ and the noises are independent of $f$, and each $\bm x_i$ is a function of $\mathcal{D}_{i-1}$ and $U$, the conditional law of $f$ given $\mathcal{G}_n$ is the GP posterior~\eqref{eq:posterior_mean}-\eqref{eq:posterior_cov} computed at the realized design. Below, conditional variances given $\mathcal{G}_n$ are thus computed at a fixed design.

\paragraph{Step 1 (A distant measurable competitor).} Let $\bm x'_n$ be the vertex of $[0, 1]^d$ farthest from $\bm x_n$, which is $\mathcal{G}_n$-measurable and satisfies $\norm{\bm x'_n - \bm x_n}_2 \geq \sqrt{d}/2 \geq 1/2$. Thus $\kappa \coloneqq k_S(\bm x'_n, \bm x_n) \leq 1 - c_S$. Let $D(t) = f(\bm x'_n, t) - f(\bm x_n, t)$ be the gap between the competitor and the query at time $t$. Since $\sup_{\bm x} f(\bm x, t_n) \geq f(\bm x'_n, t_n)$ and $r_n \geq 0$,
\begin{equation} \label{eq:appB_gap}
    r_n \geq \max\left(0, D(t_n)\right),
\end{equation}
and, conditionally on $\mathcal{G}_n$, $D(t_n) \sim \mathcal{N}(\hat\mu_n, \hat\sigma^2_n)$ with $\hat\mu_n = \mu_{n-1}(\bm x'_n, t_n) - \mu_{n-1}(\bm x_n, t_n)$ and $\hat\sigma^2_n$ the posterior variance studied below.

\paragraph{Step 2 (The posterior variance of the gap is bounded from below).} Conditionally on $\mathcal{G}_n$, the design $\bm x_1, \cdots, \bm x_n, \bm x'_n$ is fixed, and $\hat\sigma^2_n$ is the conditional variance of $D(t_n)$ given $y_1, \cdots, y_{n-1}$ in the Gaussian model where this design is deterministic. We show that these observations are less informative about the gap $D(t)$ than noisy observations of the gap itself, to which Lemma~\ref{lem:szego} applies. For $n = 1$, there is no observation and the arguments below hold with empty sets.

By Assumption~\ref{ass:covariance}, expanding the covariance of the differences gives, for all $t, t' \in \mathbb{R}$,
\begin{equation*}
    \Cov\left(D(t), D(t')\right) = \left(k_S(\bm x'_n, \bm x'_n) - 2k_S(\bm x'_n, \bm x_n) + k_S(\bm x_n, \bm x_n)\right)k_T(t - t') = 2(1 - \kappa)k_T(t - t'),
\end{equation*}
where the last equality uses $k_S(\bm x, \bm x) = 1$. Since $\kappa \leq 1 - c_S < 1$, $g = D/\sqrt{2(1 - \kappa)}$ is a $\mathcal{GP}(0, k_T)$, and $X_i = g(t_i)$, $i \in \mathbb{Z}$, is exactly the sequence considered in Lemma~\ref{lem:szego}. Because $D(t_n) = \sqrt{2(1 - \kappa)}X_n$, we have
\begin{equation} \label{eq:appB_var_Xn}
    \hat\sigma^2_n = 2(1 - \kappa)\Var\left[X_n | y_1, \cdots, y_{n-1}\right].
\end{equation}

We decompose each spatial slice $f(\bm x_i, \cdot)$, $i < n$, into a component proportional to $g$ and a component uncorrelated with $g$. By Assumption~\ref{ass:covariance}, for all $t, t' \in \mathbb{R}$,
\begin{equation} \label{eq:appB_cov_fg}
    \Cov\left(f(\bm x_i, t), g(t')\right) = \frac{\Cov\left(f(\bm x_i, t), f(\bm x'_n, t')\right) - \Cov\left(f(\bm x_i, t), f(\bm x_n, t')\right)}{\sqrt{2(1 - \kappa)}} = \zeta_ik_T(t - t'),
\end{equation}
where $\zeta_i = \left(k_S(\bm x_i, \bm x'_n) - k_S(\bm x_i, \bm x_n)\right)/\sqrt{2(1 - \kappa)}$. Let $H_i(t) = f(\bm x_i, t) - \zeta_ig(t)$. By~\eqref{eq:appB_cov_fg} and $\Cov(g(t), g(t')) = k_T(t - t')$, we have, for all $t, t' \in \mathbb{R}$,
\begin{equation*}
    \Cov\left(H_i(t), g(t')\right) = \Cov\left(f(\bm x_i, t), g(t')\right) - \zeta_i\Cov\left(g(t), g(t')\right) = \zeta_ik_T(t - t') - \zeta_ik_T(t - t') = 0.
\end{equation*}
Since $\bm H = \left(H_i(t_i)\right)_{i < n}$ and $(X_i)_{i \in \mathbb{Z}}$ are jointly Gaussian (their entries are linear combinations of values of $f$) and uncorrelated, they are independent. As the noise is independent of $f$, the observations $y_1, \ldots, y_{n-1}$ are
\begin{equation} \label{eq:appB_obs}
    y_i = \zeta_iX_i + H_i(t_i) + \epsilon_i, \quad i < n,
\end{equation}
with $(X_i)_{i \in \mathbb{Z}}$, $\bm H$ and $(\epsilon_i)_{i < n}$ mutually independent. Finally, taking $t = t' = t_i$ in~\eqref{eq:appB_cov_fg} shows that $\zeta_i = \Cov\left(f(\bm x_i, t_i), X_i\right)$ is the correlation between two variables of variance $1$, so that $|\zeta_i| \leq 1$ by the Cauchy--Schwarz inequality.

We now show that the observations~\eqref{eq:appB_obs} are less informative about $X_n$ than noisy observations of the $X_i$ themselves. Let $Y_i = X_i + \epsilon'_i$, $i \in \mathbb{Z}$, where the $\epsilon'_i \sim \mathcal{N}(0, \sigma^2_0)$ are i.i.d.\ and independent of everything else, so that $(X_i)_{i \in \mathbb{Z}}$ and $(Y_i)_{i \in \mathbb{Z}}$ are exactly the sequences considered in Lemma~\ref{lem:szego}. For a Gaussian vector $(Z, \bm V)$ with $Z$ scalar, we recall three basic properties:~(a)~$\Var[Z | \bm V]$ can only decrease when $\bm V$ is replaced by a Gaussian vector of which $\bm V$ is a linear function, (b)~it does not change when variables independent of $(Z, \bm V)$ are appended to $\bm V$, and (c)~it only depends on the law of $(Z, \bm V)$. Let $\xi_1, \cdots, \xi_{n-1} \sim \mathcal{N}(0, 1)$ be i.i.d.\ and independent of everything else. Then,
\begin{align}
    \Var\left[X_n | y_1, \cdots, y_{n-1}\right] &\geq \Var\left[X_n | y_1, \cdots, y_{n-1}, \bm H\right] = \Var\left[X_n | \left(\zeta_iX_i + \epsilon_i\right)_{i < n}\right] \label{eq:appB_remove_H}\\
    &= \Var\left[X_n | \left(\zeta_iY_i + \sqrt{1 - \zeta_i^2}\,\sigma_0\xi_i\right)_{i < n}\right] \label{eq:appB_garbling}\\
    &\geq \Var\left[X_n | \left(Y_i, \xi_i\right)_{i < n}\right] = \Var\left[X_n | Y_1, \cdots, Y_{n-1}\right]. \label{eq:appB_less_info}
\end{align}
In~\eqref{eq:appB_remove_H}, the inequality follows from~(a). The equality follows from~(b), since observing $(y_i)_{i < n}$ and $\bm H$ amounts to observing $\left(\zeta_iX_i + \epsilon_i\right)_{i < n} = \left(y_i - H_i(t_i)\right)_{i < n}$ and $\bm H$, which is independent of $\left((X_i)_{i \in \mathbb{Z}}, (\epsilon_i)_{i < n}\right)$. Equality~\eqref{eq:appB_garbling} follows from~(c): since $|\zeta_i| \leq 1$, $\zeta_iY_i + \sqrt{1 - \zeta_i^2}\,\sigma_0\xi_i = \zeta_iX_i + \left(\zeta_i\epsilon'_i + \sqrt{1 - \zeta_i^2}\,\sigma_0\xi_i\right)$, where the noises $\zeta_i\epsilon'_i + \sqrt{1 - \zeta_i^2}\,\sigma_0\xi_i$ have exactly the same distribution as $\epsilon_i$, since $\mathcal{N}\left(0, \zeta_i^2\sigma^2_0 + (1 - \zeta_i^2)\sigma^2_0\right) = \mathcal{N}(0, \sigma^2_0)$ and they are independent of $(X_i)_{i \in \mathbb{Z}}$. Finally, the inequality and the equality in~\eqref{eq:appB_less_info} follow from~(a) and~(b), respectively.

The covariances $\Cov(X_i, X_j) = \Cov(X_i, Y_j) = k_T((i - j)\Delta)$ and $\Cov(Y_i, Y_j) = k_T((i - j)\Delta) + \sigma^2_0\indic\{i = j\}$ only depend on $i - j$, so that shifting all indices by $n$ does not change the law of $(X_n, Y_1, \cdots, Y_{n-1})$. By Lemma~\ref{lem:szego}, this gives $\Var\left[X_n | Y_1, \cdots, Y_{n-1}\right] = \Var\left[X_0 | Y_{-1}, \cdots, Y_{-(n-1)}\right] = e_{n-1} \geq e_\infty$. Combining this with~\eqref{eq:appB_var_Xn} and~\eqref{eq:appB_remove_H}--\eqref{eq:appB_less_info}, and using $1 - \kappa \geq c_S$, we obtain
\begin{equation} \label{eq:appB_var_bound}
    \hat\sigma^2_n \geq 2(1 - \kappa)e_\infty \geq 2c_Se_\infty = \sigma^2_\star.
\end{equation}

\paragraph{Step 3 (The posterior mean of the gap is bounded from below).} Since $\bm x_n, \bm x'_n \in \mathcal{S}$, we have $D(t_n) = f(\bm x'_n, t_n) - f(\bm x_n, t_n) \geq -\range(f(\cdot, t_n))$. Since $\bm x_n$ and $\bm x'_n$ are $\mathcal{G}_n$-measurable, taking conditional expectations gives
\begin{equation*}
    \hat\mu_n = \mathbb{E}\left[D(t_n) | \mathcal{G}_n\right] \geq -M_n, \quad \text{with } M_n = \mathbb{E}\left[\range(f(\cdot, t_n)) | \mathcal{G}_n\right].
\end{equation*}
Since $f(\cdot, t_n)$ has the same law as $g_S$, $\mathbb{E}[M_n] = \mathbb{E}\left[\range(f(\cdot, t_n))\right] = \bar R$. Note that $\bar R > 0$, since $g_S$ is not constant by Assumption~\ref{ass:covariance}. By Markov's inequality, the event $\mathcal{A}_n = \left\{M_n \leq 2\bar R\right\}$ thus has probability at least $1/2$, and $\hat\mu_n \geq -2\bar R$ on $\mathcal{A}_n$.

\paragraph{Conclusion.} Let $Z \sim \mathcal{N}(0, 1)$, $\bar\mu \in \mathbb{R}$ and $\bar\sigma > 0$. It is well-known that
\begin{equation} \label{eq:appB_positive_part}
    \mathbb{E}\left[\max\left(0, \bar\mu + \bar\sigma Z\right)\right] = \bar\mu\Phi\left(\bar\mu/\bar\sigma\right) + \bar\sigma\phi\left(\bar\mu/\bar\sigma\right) = \bar\sigma\Psi\left(\bar\mu/\bar\sigma\right).
\end{equation}
It is also known that the RHS of~\eqref{eq:appB_positive_part} is strictly positive and increasing in both $\bar\mu$ and $\bar\sigma$. Therefore, the constant $c = \sigma_\star\Psi(-2\bar R/\sigma_\star)/2$ of Theorem~\ref{thm:lower_regret_bound} is strictly positive. On the event $\mathcal{A}_n$, $\hat\mu_n \geq -2\bar R$ and $\hat\sigma_n \geq \sigma_\star$ by~\eqref{eq:appB_var_bound}, so that~\eqref{eq:appB_gap} and~\eqref{eq:appB_positive_part} give $\mathbb{E}[r_n|\mathcal{G}_n] \geq \hat\sigma_n\Psi(\hat\mu_n/\hat\sigma_n) \geq \sigma_\star \Psi(-2\bar R/\sigma_\star)$. Since $r_n \geq 0$, taking expectations yields $\mathbb{E}[r_n] \geq \mathbb{P}[\mathcal{A}_n]\,\sigma_\star \Psi(-2\bar R/\sigma_\star) \geq c$. Summing over $n$ gives $\mathbb{E}[R_n] \geq cn$, and $\mathbb{E}[R_n] \leq n\bar R$ because $r_i \leq \range(f(\cdot, t_i))$, so that $\mathbb{E}[R_n] \in \Theta(n)$.

Finally, $\mathbb{E}[R_n^2] \leq n\sum_{i=1}^n\mathbb{E}[r_i^2] \leq n^2\mathbb{E}[\range(g_S)^2]$ by the Cauchy--Schwarz inequality. We use Paley--Zygmund inequality, which states that $\mathbb{P}[V > \mathbb{E}[V]/2] \geq \mathbb{E}[V]^2/(4\mathbb{E}[V^2])$ for any nonnegative random variable $V$ with finite second moment, on $R_n$. We get $\mathbb{P}[R_n \geq cn/2] \geq \mathbb{P}[R_n > \mathbb{E}[R_n]/2] \geq \mathbb{E}[R_n]^2/(4\mathbb{E}[R_n^2]) \geq c^2/(4\mathbb{E}[\range(g_S)^2])$. \qed

\section{Proof of the Information-Gain Barrier} \label{app:barrier}

Fix $\bm x \in \mathcal{S}$ and consider the design $\bm x_1 = \cdots = \bm x_n = \bm x$. Since $k_S(\bm x, \bm x) = 1$, $\bm K^{(n)} = \left(k_T((i-j)\Delta)\right)_{i,j \in [n]}$, and the observations $y_i = X_i + \epsilon_i$, with $X_i = f(\bm x, t_i)$, form the stationary sequence $(Y_i)$ of Lemma~\ref{lem:szego}, whose covariance matrix is $\sigma^2_0\bm I + \bm K^{(n)}$. For $i \in [n]$, let $\bm\Sigma_i$ be the covariance matrix of $(Y_1, \cdots, Y_i)$ and $\bm c_i = \left(\Cov(Y_j, Y_i)\right)_{j < i}$. The determinant formula for block matrices gives $\det\bm\Sigma_i = \det\bm\Sigma_{i-1}\left(\Var[Y_i] - \bm c_i^\top\bm\Sigma_{i-1}^{-1}\bm c_i\right)$, and the Schur complement $\Var[Y_i] - \bm c_i^\top\bm\Sigma_{i-1}^{-1}\bm c_i$ is the conditional variance $\Var\left[Y_i | Y_1, \cdots, Y_{i-1}\right]$ of the Gaussian vector $(Y_1, \cdots, Y_i)$~\citep[Appendix~A.2]{williams2006gaussian}. By induction on $i$,
\begin{equation*}
    \det\left(\sigma^2_0\bm I + \bm K^{(n)}\right) = \prod_{i=1}^n \Var\left[Y_i | Y_1, \cdots, Y_{i-1}\right] = \prod_{i=1}^n \left(e_{i-1} + \sigma^2_0\right) \geq \left(\sigma^2_0e^{2\rho}\right)^n,
\end{equation*}
where the second equality uses stationarity and the identity $\Var[Y_0 | Y_{-1}, \cdots, Y_{-m}] = e_m + \sigma^2_0$ (see the proof of Lemma~\ref{lem:szego}), and the inequality uses $e_{i-1} \geq e_\infty = \sigma^2_0(e^{2\rho} - 1)$ (Lemma~\ref{lem:szego}). Hence $\gamma_n(k) \geq \frac{1}{2}\log\det\left(\bm I + \sigma_0^{-2}\bm K^{(n)}\right) \geq \rho n$. Conversely, $\gamma_n(k) \leq \frac{n}{2}\log(1 + \sigma_0^{-2})$ because, for any design, the same identity yields $\frac{1}{2}\log\det(\bm I + \sigma_0^{-2}\bm K^{(n)}) = \frac{1}{2}\sum_{i=1}^n\log(1 + \sigma_0^{-2}\sigma^2_{i-1}(\bm x_i, t_i))$~\citep[Lemma~5.3]{gpucb}, and $\sigma^2_{i-1} \leq 1$. Thus $\gamma_n(k) \in \Theta(n)$ when $\rho > 0$, i.e., when $S_T^c \neq 0$ (Lemma~\ref{lem:szego}). \qed

\section{Proofs for Rhythmic Temporal Kernels} \label{app:upper_bound}

\subsection{Proof of Lemma~\ref{lem:info_gain_atomic}} \label{app:info_gain_atomic}

We write $\bm A \preceq \bm B$ when $\bm B - \bm A$ is positive semidefinite, and we use the following inequality that we prove for the sake of completeness: for positive semidefinite $\bm A, \bm B$,
\begin{equation} \label{eq:logdet_subadditive}
    \log\det(\bm I + \bm A + \bm B) \leq \log\det(\bm I + \bm A) + \log\det(\bm I + \bm B).
\end{equation}
Indeed, since $\bm I + \bm A$ is positive definite, we can factor it out of $\bm I + \bm A + \bm B$. With $\bm X = (\bm I + \bm A)^{-1/2}\bm B^{1/2}$, this gives
\begin{equation*}
    \log\det(\bm I + \bm A + \bm B) - \log\det(\bm I + \bm A) = \log\det\left(\bm I + \bm X\bm X^\top\right) = \log\det\left(\bm I + \bm X^\top\bm X\right) = \log\det\left(\bm I + \bm N\right),
\end{equation*}
where the second equality is Sylvester's determinant identity and $\bm N = \bm X^\top\bm X = \bm B^{1/2}(\bm I + \bm A)^{-1}\bm B^{1/2}$. Since $\bm A \succeq \bm 0$, the eigenvalues of $(\bm I + \bm A)^{-1}$ are at most $1$, i.e., $(\bm I + \bm A)^{-1} \preceq \bm I$. Hence, for all $\bm v$, $\bm v^\top\bm N\bm v = \bm w^\top(\bm I + \bm A)^{-1}\bm w \leq \bm w^\top\bm w = \bm v^\top\bm B\bm v$ with $\bm w = \bm B^{1/2}\bm v$, i.e., $\bm 0 \preceq \bm N \preceq \bm B$. It remains to relate the order $\preceq$ to $\log\det$. Let $\lambda_1(\bm M) \geq \lambda_2(\bm M) \geq \cdots$ denote the eigenvalues of a symmetric matrix $\bm M$. A well-known consequence of the Courant--Fischer min-max theorem is that $\bm N \preceq \bm B$ implies $\lambda_k(\bm N) \leq \lambda_k(\bm B)$ for all $k$. Therefore, $\log\det(\bm I + \bm N) = \sum_k \log\left(1 + \lambda_k(\bm N)\right) \leq \sum_k \log\left(1 + \lambda_k(\bm B)\right) = \log\det(\bm I + \bm B)$, which proves~\eqref{eq:logdet_subadditive}. We have shown along the way that $\bm 0 \preceq \bm N \preceq \bm B$ implies $\log\det(\bm I + \bm N) \leq \log\det(\bm I + \bm B)$, which we use again below.

\begin{proof}[Proof of Lemma~\ref{lem:info_gain_atomic}]
    Fix a design $\bm x_1, \cdots, \bm x_n$, let $\bm K_S = \left(k_S(\bm x_i, \bm x_j)\right)_{i,j}$ and, with the notations of~\eqref{eq:moving_mixture}, $\bm D_m = \mathrm{diag}\left(\psi_m(t_1), \cdots, \psi_m(t_n)\right)$. By Assumption~\ref{ass:covariance}, denoting by $\odot$ the Hadamard product,
    \begin{equation*}
        \bm K^{(n)} = \bm K_S \odot \left(k_T(t_i - t_j)\right)_{i,j} = \sum_{m \geq 0} w_m\bm D_m\bm K_S\bm D_m = \sum_{m \leq 2P} w_m\bm D_m\bm K_S\bm D_m + \bm K_{>P},
    \end{equation*}
    where all terms are positive semidefinite and the series converges since $\sum_m w_m = 1$ and $|\psi_m| \leq 1$. Applying~\eqref{eq:logdet_subadditive} multiple times yields
    \begin{equation*}
        \log\det\left(\bm I + \sigma_0^{-2}\bm K^{(n)}\right) \leq \sum_{m \leq 2P}\log\det\left(\bm I + \sigma_0^{-2}w_m\bm D_m\bm K_S\bm D_m\right) + \log\det\left(\bm I + \sigma_0^{-2}\bm K_{>P}\right).
    \end{equation*}
    We now show that each term $w_m\bm D_m\bm K_S\bm D_m$ is dominated by $\bm K_S$. First, $\bm D_m$ is diagonal with entries $\psi_m(t_i) \in [-1, 1]$ and $0 \leq w_m \leq \sum_{m'} w_{m'} = 1$, so that $w_m\bm D_m^2$ is diagonal with entries in $[0, 1]$, i.e., $w_m\bm D_m^2 \preceq \bm I$. Second, with $\bm Y = \sqrt{w_m}\,\bm D_m\bm K_S^{1/2}$, we have $w_m\bm D_m\bm K_S\bm D_m = \bm Y\bm Y^\top$ and, by Sylvester's determinant identity,
    \begin{equation*}
        \log\det\left(\bm I + \sigma_0^{-2}w_m\bm D_m\bm K_S\bm D_m\right) = \log\det\left(\bm I + \sigma_0^{-2}\bm Y^\top\bm Y\right), \quad \text{with } \bm Y^\top\bm Y = \bm K_S^{1/2}w_m\bm D_m^2\bm K_S^{1/2}.
    \end{equation*}
    For all $\bm v$ and $\bm u = \bm K_S^{1/2}\bm v$, $w_m\bm D_m^2 \preceq \bm I$ gives $\bm v^\top\bm Y^\top\bm Y\bm v = \bm u^\top w_m\bm D_m^2\bm u \leq \bm u^\top\bm u = \bm v^\top\bm K_S\bm v$, so that $\bm 0 \preceq \bm Y^\top\bm Y \preceq \bm K_S$. Therefore, each of the $2P + 1$ terms of the sum is at most $\log\det(\bm I + \sigma_0^{-2}\bm K_S)$, which is at most $2\gamma_n(k_S)$ since $\gamma_n(k_S)$ is the maximum of $\frac{1}{2}\log\det(\bm I + \sigma_0^{-2}\bm K_S)$ over all designs. Moreover, $\log\det(\bm I + \bm B) = \sum_k \log(1 + \lambda_k(\bm B)) \leq \sum_k \lambda_k(\bm B) = \Tr(\bm B)$ for any positive semidefinite $\bm B$, since $\log(1 + \lambda) \leq \lambda$, and, since $\cos^2 + \sin^2 = 1$ and $k_S(\bm x, \bm x) = 1$,
    \begin{equation*}
        \Tr\left(\bm K_{>P}\right) = \sum_{i=1}^n\sum_{p > P} 2\alpha_p\left(\cos^2(2\pi\omega_pt_i) + \sin^2(2\pi\omega_pt_i)\right) = n\varepsilon_P.
    \end{equation*}
    Combining these bounds, we obtain, for the fixed design,
    \begin{equation*}
        \log\det\left(\bm I + \sigma_0^{-2}\bm K^{(n)}\right) \leq 2(2P + 1)\gamma_n(k_S) + \sigma_0^{-2}\Tr\left(\bm K_{>P}\right) = 2(2P + 1)\gamma_n(k_S) + \frac{n\varepsilon_P}{\sigma^2_0}.
    \end{equation*}
    After dividing by $2$, the left-hand side is the information gain of the design, whereas the right-hand side does not depend on the design. Therefore, \eqref{eq:info_gain_atomic}~is also an upper-bound of the maximum information gain (i.e., the information gain maximized over the design). Finally, if $k_T$ is a sum of $L$ cosines, then $\alpha_p = 0$ for all $p > L$, so that $\varepsilon_L = 2\sum_{p > L}\alpha_p = 0$, and~\eqref{eq:info_gain_atomic} with $P = L$ gives $\gamma_n(k) \leq (2L + 1)\gamma_n(k_S)$.
\end{proof}

\subsection{GP-UCB in the Time-Varying Setting and Proof of Theorem~\ref{thm:upper_regret_bound}}

The following lemma bounds the regret of GP-UCB in terms of the maximum information gain $\gamma_n(k)$ of the spatio-temporal kernel, for any temporal kernel. Its proof follows the proof of Theorem~2 of~\citet{gpucb}. It is also very similar to the analysis of TV-GP-UCB by~\citet{bogunovic2016time}, who apply the same argument to their temporal kernel $k_T(t_i - t_j) = (1 - \varepsilon)^{|i - j|/2}$. This argument does not depend on some particular expression of $k_T$, because it only uses the Gaussian posterior of the spatio-temporal GP $f$ and the fact that each spatial slice $f(\cdot, t_i)$ is a $\mathcal{GP}(0, k_S)$.

\begin{lemma} \label{lem:gpucb_tv}
    Under Assumptions~\ref{ass:surrogate}-\ref{ass:lipschitz}, let $\delta \in (0, 1)$ and let GP-UCB use the exploration parameters
    \begin{equation*}
        \beta_i = 2\log\left(\frac{\pi^2 i^2}{2\delta}\right) + 2d\log\ceil*{i^2 d b\sqrt{\log\left(\frac{\pi^2 d a i^2}{2\delta}\right)}}, \quad i \in \mathbb{N}^*,
    \end{equation*}
    of Theorem~\ref{thm:upper_regret_bound}, where $a \geq 1$ and $b > 0$ are the constants of Assumption~\ref{ass:lipschitz}. Then, with probability at least $1 - \delta$, $R_n \leq \sqrt{C_1n\beta_n\gamma_n(k)} + \pi^2/6$ for all $n \in \mathbb{N}^*$, where $C_1 = 8/\log(1 + \sigma_0^{-2})$.
\end{lemma}

\begin{proof}
    For all $i \in \mathbb{N}^*$, let $f_i = f(\cdot, t_i)$ be the objective at time $t_i$, and let $\mu_{i-1}(\cdot) = \mu_{i-1}(\cdot, t_i)$ and $\sigma_{i-1}(\cdot) = \sigma_{i-1}(\cdot, t_i)$ be its posterior mean and standard deviation given the first $i - 1$ observations gathered in $\mathcal{D}_{i-1}$. Since $\Cov(f_i(\bm x), f_i(\bm x')) = k_S(\bm x, \bm x')k_T(0) = k_S(\bm x, \bm x')$, each $f_i$ is a $\mathcal{GP}(0, k_S)$, so that Assumption~\ref{ass:lipschitz} applies to $f_i$. By definition of GP-UCB, $\bm x_i$ maximizes $\mu_{i-1} + \beta_i^{1/2}\sigma_{i-1}$ over $\mathcal{S}$. Let $\bm x^*_i$ be a maximizer of $f_i$ over $\mathcal{S}$, so that $r_i = f_i(\bm x^*_i) - f_i(\bm x_i)$.

    We first carry out the analysis with the exploration parameters $\beta_i$ of GP-UCB, positive numbers $u_i$ and positive integers $N_i$ left unspecified, and we then show that the choice of $\beta_i$ in the statement makes this analysis hold with probability at least $1 - \delta$. Let $\mathcal{S}_i \subset \mathcal{S}$ be the set of the $N_i^d$ centers of a regular partition of $[0, 1]^d$ into cubes of side $1/N_i$. Every $\bm x \in \mathcal{S}$ belongs to one of these cubes, whose center $[\bm x]_i \in \mathcal{S}_i$ satisfies $\norm{\bm x - [\bm x]_i}_1 \leq d/(2N_i)$. Consider the events
    \begin{align*}
        E_1 &= \left\{\forall i \in \mathbb{N}^*, \forall \bm x \in \mathcal{S}_i : |f_i(\bm x) - \mu_{i-1}(\bm x)| \leq \beta_i^{1/2}\sigma_{i-1}(\bm x)\right\},\\
        E_2 &= \left\{\forall i \in \mathbb{N}^* : |f_i(\bm x_i) - \mu_{i-1}(\bm x_i)| \leq \beta_i^{1/2}\sigma_{i-1}(\bm x_i)\right\},\\
        E_3 &= \left\{\forall i \in \mathbb{N}^*, \forall j \in [d] : \sup_{\bm x \in \mathcal{S}}|\partial f_i(\bm x)/\partial x_j| \leq u_i\right\},
    \end{align*}
    and let $E = E_1 \cap E_2 \cap E_3$.

    On $E_3$, the mean value theorem gives $|f_i(\bm x) - f_i(\bm x')| \leq u_i\norm{\bm x - \bm x'}_1$ for all $\bm x, \bm x' \in \mathcal{S}$, so that $f_i(\bm x^*_i) \leq f_i([\bm x^*_i]_i) + u_id/(2N_i)$. Therefore, on $E$,
    \begin{equation} \label{eq:appD_inst}
        r_i \leq f_i([\bm x^*_i]_i) + \frac{u_id}{2N_i} - f_i(\bm x_i) \leq \mu_{i-1}([\bm x^*_i]_i) + \beta_i^{1/2}\sigma_{i-1}([\bm x^*_i]_i) + \frac{u_id}{2N_i} - f_i(\bm x_i) \leq 2\beta_i^{1/2}\sigma_{i-1}(\bm x_i) + \frac{u_id}{2N_i},
    \end{equation}
    where the second inequality holds on $E_1$, and the third one holds because $\bm x_i$ maximizes $\mu_{i-1} + \beta_i^{1/2}\sigma_{i-1}$ over $\mathcal{S} \supseteq \mathcal{S}_i$ and because $\mu_{i-1}(\bm x_i) - f_i(\bm x_i) \leq \beta_i^{1/2}\sigma_{i-1}(\bm x_i)$ on $E_2$. For the discretization errors $u_id/(2N_i)$ to be summable, we take $N_i = \ceil{i^2du_i}$, so that $u_id/(2N_i) \leq 1/i^2$ and, on $E$, $r_i \leq 2\beta_i^{1/2}\sigma_{i-1}(\bm x_i) + 1/i^2$ for all $i \in \mathbb{N}^*$.

    Let $\mathcal{G}_i$ be the $\sigma$-algebra generated by $\mathcal{D}_{i-1}$. As shown in Appendix~\ref{app:lower_bound}, $f_i(\bm x) \sim \mathcal{N}(\mu_{i-1}(\bm x), \sigma^2_{i-1}(\bm x))$ conditionally on $\mathcal{G}_i$ for any $\mathcal{G}_i$-measurable $\bm x$. Since $\mathbb{P}[|Z| > c] \leq e^{-c^2/2}$ for $Z \sim \mathcal{N}(0, 1)$ and $c > 0$, union bounds over $i$ and $\bm x \in \mathcal{S}_i$ give $\mathbb{P}[E_1^c] \leq \sum_i N_i^de^{-\beta_i/2}$ and $\mathbb{P}[E_2^c] \leq \sum_i e^{-\beta_i/2}$. By Assumption~\ref{ass:lipschitz} and a union bound over $i$ and $j$, $\mathbb{P}[E_3^c] \leq \sum_i dae^{-(u_i/b)^2}$. Since $\sum_{i \geq 1} 2\delta/(\pi^2i^2) = \delta/3$ and $N_i \geq 1$, we thus have $\mathbb{P}[E] \geq 1 - \delta$ as soon as, for all $i \in \mathbb{N}^*$,
    \begin{equation} \label{eq:appD_conditions}
        dae^{-(u_i/b)^2} \leq \frac{2\delta}{\pi^2i^2} \quad \text{and} \quad N_i^de^{-\beta_i/2} \leq \frac{2\delta}{\pi^2i^2}.
    \end{equation}
    The first condition holds with equality for $u_i = b\sqrt{\log(\pi^2dai^2/(2\delta))}$. With this choice, $N_i = \ceil{i^2db\sqrt{\log(\pi^2dai^2/(2\delta))}}$, and the second condition reads
    \begin{equation*}
        \beta_i \geq 2\log\left(\frac{\pi^2i^2}{2\delta}\right) + 2d\log N_i = 2\log\left(\frac{\pi^2 i^2}{2\delta}\right) + 2d\log\ceil*{i^2 d b\sqrt{\log\left(\frac{\pi^2 d a i^2}{2\delta}\right)}},
    \end{equation*}
    which justifies the value of $\beta_i$ of the statement of Lemma~\ref{lem:gpucb_tv}.

    Summing over $i \leq n$, and using $\sum_i 1/i^2 \leq \pi^2/6$, the fact that $\beta_i$ is non-decreasing in $i$ and the Cauchy--Schwarz inequality, we obtain the well-known bound
    \begin{equation*}
        R_n \leq 2\beta_n^{1/2}\sum_{i=1}^n\sigma_{i-1}(\bm x_i) + \frac{\pi^2}{6} \leq 2\sqrt{n\beta_n\sum_{i=1}^n\sigma^2_{i-1}(\bm x_i)} + \frac{\pi^2}{6}.
    \end{equation*}
    The function $s \mapsto \log(1 + \sigma_0^{-2}s)$ is concave, vanishes at $s = 0$ and equals $\log(1 + \sigma_0^{-2})$ at $s = 1$, so that $s \leq \log(1 + \sigma_0^{-2}s)/\log(1 + \sigma_0^{-2})$ for all $s \in [0, 1]$. Since $\sigma^2_{i-1} \leq k((\bm x, t), (\bm x, t)) = 1$, this gives
    \begin{equation*}
        \sum_{i=1}^n\sigma^2_{i-1}(\bm x_i) \leq \frac{\sum_{i=1}^n\log\left(1 + \sigma_0^{-2}\sigma^2_{i-1}(\bm x_i)\right)}{\log(1 + \sigma_0^{-2})} = \frac{\log\det\left(\bm I + \sigma_0^{-2}\bm K^{(n)}\right)}{\log(1 + \sigma_0^{-2})} \leq \frac{2\gamma_n(k)}{\log(1 + \sigma_0^{-2})},
    \end{equation*}
    where $\bm K^{(n)}$ is computed at the realized design, the equality is the identity of Appendix~\ref{app:barrier}, also used in~\citet[Lemma~5.3]{gpucb}, and the last inequality holds by definition of $\gamma_n(k)$. Plugging this bound into the previous one yields $R_n \leq \sqrt{8n\beta_n\gamma_n(k)/\log(1 + \sigma_0^{-2})} + \pi^2/6 = \sqrt{C_1n\beta_n\gamma_n(k)} + \pi^2/6$ on $E$, for all $n \in \mathbb{N}^*$.
\end{proof}

\begin{proof}[Proof of Theorem~\ref{thm:upper_regret_bound}]
    The bound~\eqref{eq:upper_bound} follows from Lemmas~\ref{lem:gpucb_tv} and~\ref{lem:info_gain_atomic}, minimized over $P$. Assume $\gamma_n(k_S) \in \mathcal{O}(n^\eta)$ with $\eta < 1$ and $\varepsilon_P \in o(1/\log P)$, and take $P = \ceil{n^c}$ with $0 < c < 1 - \eta$. Since $\beta_n \in \mathcal{O}(\log n)$, $(2P+1)\gamma_n(k_S)\beta_n \in \mathcal{O}(n^{c+\eta}\log n) \subseteq o(n)$ and $n\varepsilon_P\beta_n \in n\cdot o(1/\log n^c)\cdot\mathcal{O}(\log n) = o(n)$. Hence $\beta_n\Gamma_n \in o(n)$ and $R_n \in o(n)$.
\end{proof}

\subsection{Proof of Corollary~\ref{cor:rates}} \label{app:rates}

For an RBF $k_S$, $\gamma_n(k_S) \in \mathcal{O}((\log n)^{d+1})$~\citep{gpucb}, and $\beta_n \in \mathcal{O}(\log n)$. (i)~Take $P = L$, so that $\Gamma_n \leq (2L+1)\gamma_n(k_S)$. (ii)~Take $P = \ceil{\zeta^{-1}\log n}$, so that $n\varepsilon_P \leq C$ and $\Gamma_n \leq (2\zeta^{-1}\log n + 3)\gamma_n(k_S) + C/(2\sigma^2_0)$. For the periodic kernel, $\alpha_p = e^{-z}I_p(z)$ with $z = \ell^{-2}$ (Appendix~\ref{app:examples}), and, since $(k+p)! \geq k!\,p!$,
\begin{equation*}
    I_p(z) = \sum_{k \geq 0}\frac{(z/2)^{2k+p}}{k!\,(k+p)!} \leq \frac{(z/2)^p}{p!}e^{z^2/4},
\end{equation*}
so that $\varepsilon_P \leq 2e^{z^2/4 - z}\sum_{p > P}(z/2)^p/p!$ decays faster than any exponential. (iii)~Take $P = \ceil{(n/\gamma_n(k_S))^{1/(1+\zeta)}}$, so that $\Gamma_n \in \mathcal{O}(n^{1/(1+\zeta)}\gamma_n(k_S)^{\zeta/(1+\zeta)})$. Plugging these bounds in~\eqref{eq:upper_bound} yields the rates. \qed

\section{Experimental Details} \label{app:experiments}

\subsection{Setup} \label{app:setup}

\paragraph{Objectives.} Time is measured in sampling periods ($\Delta = 1$), so that a band-limited kernel with $\supp(S_T) = [-\tau, \tau]$ is sampled above its Nyquist rate if and only if $2\tau < 1$. The spatial domain $\mathcal{S} = [0, 1]$ is a regular grid of $100$ points, and $k_S$ is an RBF kernel with lengthscale $0.1$. For each temporal kernel and each seed, one objective is sampled exactly from $\mathcal{GP}(0, k_Sk_T)$ on the grid and for $t_1, \cdots, t_{800}$, as $\bm A_S\bm Z\bm A_T^\top$, where $\bm Z$ has i.i.d.\ standard normal entries and $\bm A_S\bm A_S^\top$, $\bm A_T\bm A_T^\top$ are the spatial and temporal covariance matrices. The noise variance is $\sigma^2_0 = 0.01$. The temporal kernels of Figure~\ref{fig:regret} are listed in Table~\ref{tab:results}. The periodic kernel has period $r = 40$ and $\ell = 1$. The sum of cosines is $0.1 + 0.4\cos(2\pi t/40) + 0.3\cos(2\pi t/15) + 0.2\cos(2\pi t/7)$. The mixture is $0.5$ periodic $+\ 0.5$ RBF ($\ell = 10$), and the locally periodic kernel is the product of the periodic kernel with an RBF kernel of lengthscale $50$.

\paragraph{Algorithms.} All algorithms query the grid and use the exact posterior~\eqref{eq:posterior_mean}-\eqref{eq:posterior_cov} of the true prior, with a Cholesky factor updated at each step ($\mathcal{O}(i^2m)$ operations at step $i$ for $m$ grid points):
\begin{itemize}[itemsep=0pt, topsep=2pt]
    \item GP-UCB with $\beta = 2$, fixed to prevent the well-known over-exploration of the algorithm in practice;
    \item expected improvement (EI) over the largest posterior mean at the current time;
    \item random search, with uniform random queries.
\end{itemize}

The constant $\beta = 2$ is not covered by Theorem~\ref{thm:upper_regret_bound}, whose guarantee requires an exploration parameter $\beta_i$ that grows logarithmically with $i$~\eqref{eq:def_beta}. To test Theorem~\ref{thm:upper_regret_bound} on the algorithm it analyzes, Figure~\ref{fig:theorem_d} reports GP-UCB with the finite-domain schedule $\beta_i = 2\log(100\,i^2\pi^2/(6\delta))$, $\delta = 0.1$~\citep{gpucb}, for which a finite-domain analogue of Theorem~\ref{thm:upper_regret_bound} holds on the grid, on the two rhythmic kernels and on the RBF kernel ($\ell = 10$) of Table~\ref{tab:results}. All algorithms share the same objectives (common random numbers).

\paragraph{Compute.} One run of GP-UCB takes about $1.7$ seconds on a single CPU core. All experiments (including the Monte Carlo estimates below) take about 10 minutes on an 8-core laptop, using NumPy, SciPy and Matplotlib only.

\subsection{Full-Slice Oracle} \label{app:oracle}

Let $\mathcal{S}_m \subset \mathcal{S}$ be the grid of $m$ points, and let $\mathcal{H}_n$ be the $\sigma$-algebra generated by the noisy slices $\tilde y(\bm x, t_j) = f(\bm x, t_j) + \tilde\epsilon(\bm x, t_j)$, $\bm x \in \mathcal{S}_m$, $0 \leq j < n$, with $\tilde\epsilon(\bm x, t_j) \sim \mathcal{N}(0, \sigma^2_0)$ i.i.d.\ and independent of $f$. The full-slice oracle plays $\bm x^\star_n \in \argmax_{\bm x \in \mathcal{S}_m}\mathbb{E}[f(\bm x, t_n)|\mathcal{H}_n]$ and incurs $r^\star_n = \max_{\mathcal{S}_m}f(\cdot, t_n) - f(\bm x^\star_n, t_n)$.

\begin{proposition} \label{prop:oracle}
    Under Assumptions~\ref{ass:surrogate}-\ref{ass:sampling_freq}, for any TVBO algorithm querying $\mathcal{S}_m$, with regret measured on $\mathcal{S}_m$, $\mathbb{E}[r_n] \geq \mathbb{E}[r^\star_n]$ for all $n \in \mathbb{N}^*$.
\end{proposition}

\begin{proof}
    Run the algorithm with the observations $y_i = \tilde y(\bm x_i, t_i)$. Since $\bm x_i$ only depends on $U$ and on observations at earlier times, $\tilde\epsilon(\bm x_i, t_i)$ is independent of $f$ and of the past, and the joint law of $(f, U, (\bm x_i, y_i)_i)$ is unchanged. Then $\bm x_n$ is $\sigma(\mathcal{H}_n, U)$-measurable and, since $U$ is independent of $(f, \mathcal{H}_n)$,
    \begin{equation*}
        \mathbb{E}[r_n|\mathcal{H}_n, U] = \mathbb{E}\left[\max_{\mathcal{S}_m}f(\cdot, t_n)\Big|\mathcal{H}_n\right] - \mathbb{E}[f(\bm x_n, t_n)|\mathcal{H}_n, U] \geq \mathbb{E}\left[\max_{\mathcal{S}_m}f(\cdot, t_n)\Big|\mathcal{H}_n\right] - \max_{\bm x \in \mathcal{S}_m}\mathbb{E}[f(\bm x, t_n)|\mathcal{H}_n] = \mathbb{E}[r^\star_n|\mathcal{H}_n].
    \end{equation*}
    Taking expectations concludes the proof.
\end{proof}

\paragraph{Computation.} Let $\bm K_S = \bm V\mathrm{diag}(\bm\lambda)\bm V^\top$ be the eigendecomposition of the spatial covariance matrix on the grid, with orthonormal eigenvectors $\bm v_1, \cdots, \bm v_m$ and eigenvalues $\bm\lambda = (\lambda_1, \cdots, \lambda_m)$. By Assumption~\ref{ass:covariance}, the coordinates $c_k(t) = \bm v_k^\top f(\cdot, t)$ are independent temporal GPs with covariance $\lambda_kk_T$, and the orthonormal transform keeps the slice noises i.i.d. The oracle posterior thus factorizes into $m$ scalar Toeplitz prediction problems, which we solve exactly. Each point of Figure~\ref{fig:oracle} averages $1000$ independent samples. The noiseless curves use $\sigma^2_0 = 10^{-8}$, since exact conditioning on noiseless near-deterministic sequences is numerically ill-posed.

\subsection{Numerical Results} \label{app:numerical_results}

\begin{table}[h]
    \caption{Information rate $\rho$, asymptotic prediction error $e_\infty$, late-time regret (mean of $r_i$ over the last $25\%$ of $800$ steps, 20 objectives, $\pm$ one standard error) of each algorithm, and regret of the full-slice oracle at $n = 800$ (1000 samples), for the temporal kernels of Figure~\ref{fig:regret}.}
    \label{tab:results}
    \centering
    \footnotesize
    \setlength{\tabcolsep}{4pt}
    \begin{tabular}{lcccccc}
        \toprule
        $k_T$ & $\rho$ & $e_\infty$ & GP-UCB & EI & Random & Oracle \\
        \midrule
        RBF ($\ell = 10$) & 0.335 & 0.0096 & 0.394 $\pm$ 0.012 & 0.387 $\pm$ 0.012 & 1.517 $\pm$ 0.021 & $2.9 \cdot 10^{-3}$ \\
        Matérn-3/2 ($\ell = 10$) & 0.684 & 0.0293 & 0.472 $\pm$ 0.016 & 0.471 $\pm$ 0.020 & 1.469 $\pm$ 0.030 & $1.3 \cdot 10^{-2}$ \\
        Exponential ($\ell = 10$) & 1.495 & 0.1890 & 0.642 $\pm$ 0.023 & 0.669 $\pm$ 0.024 & 1.493 $\pm$ 0.024 & $1.6 \cdot 10^{-1}$ \\
        Sinc$^2$ ($\tau = 0.05$, above Nyquist) & 0.330 & 0.0094 & 0.482 $\pm$ 0.015 & 0.459 $\pm$ 0.012 & 1.472 $\pm$ 0.036 & $2.4 \cdot 10^{-3}$ \\
        Sinc$^2$ ($\tau = 0.7$, below Nyquist) & 2.298 & 0.9809 & 1.442 $\pm$ 0.015 & 1.406 $\pm$ 0.015 & 1.518 $\pm$ 0.009 & $1.3 \cdot 10^{0}$ \\
        0.5 periodic + 0.5 RBF & 0.293 & 0.0080 & 0.225 $\pm$ 0.008 & 0.215 $\pm$ 0.010 & 1.427 $\pm$ 0.032 & $1.9 \cdot 10^{-3}$ \\
        Periodic $\times$ RBF ($\ell = 50$) & 0.413 & 0.0128 & 0.305 $\pm$ 0.015 & 0.286 $\pm$ 0.015 & 1.473 $\pm$ 0.040 & $4.6 \cdot 10^{-3}$ \\
        Periodic ($r = 40$) & 0 & 0 & 0.017 $\pm$ 0.006 & 0.008 $\pm$ 0.001 & 1.448 $\pm$ 0.044 & $7.1 \cdot 10^{-5}$ \\
        Sum of cosines & 0 & 0 & 0.006 $\pm$ 0.000 & 0.006 $\pm$ 0.000 & 1.460 $\pm$ 0.033 & $5.5 \cdot 10^{-5}$ \\
        \bottomrule
    \end{tabular}
\end{table}

\subsection{Additional Figures} \label{app:additional_figures}

\begin{figure}[h]
    \centering
    \includegraphics[width=\textwidth]{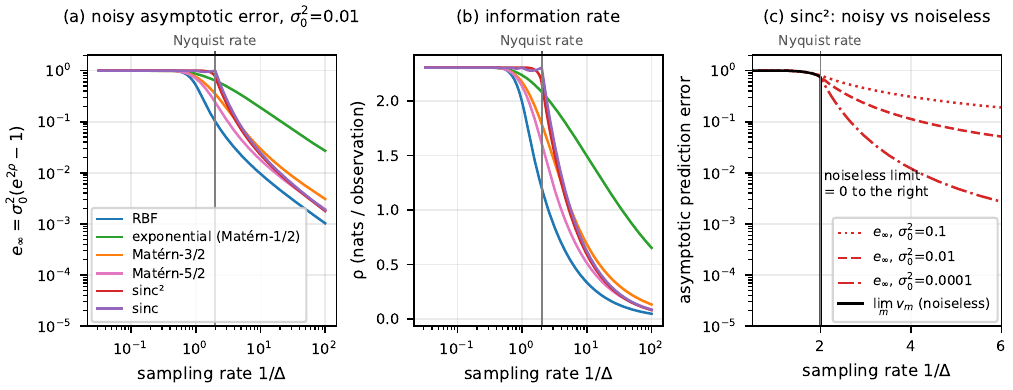}
    \caption{Asymptotic quantities as functions of the sampling frequency $1/\Delta$, for temporal kernels with unit lengthscale or cutoff frequency $\tau = 1$. Vertical lines: Nyquist rate $2\tau = 2$ of the sinc and sinc$^2$ kernels. (a)~Asymptotic prediction error $e_\infty = \sigma^2_0(e^{2\rho}-1)$, $\sigma^2_0 = 0.01$: it never vanishes. It decreases along refinements $\Delta \to \Delta/k$, $k \in \mathbb{N}^*$, but not monotonically in general. For instance, the samples of the sinc (resp., sinc$^2$) kernel are i.i.d., hence $e_\infty = 1$, when $1/\Delta = 2\tau/k$ (resp., $1/\Delta = \tau/k$). (b)~Information rate $\rho$. (c)~Sinc$^2$ kernel: noiseless asymptotic prediction error $\lim_{m \to \infty} v_m$ (Remark~\ref{rem:noiseless}), which vanishes above the Nyquist rate $2\tau = 2$, and $e_\infty$ for three noise levels, which does not.}
    \label{fig:prediction_rate}
\end{figure}

\begin{figure}[h]
    \centering
    \includegraphics[width=0.85\textwidth]{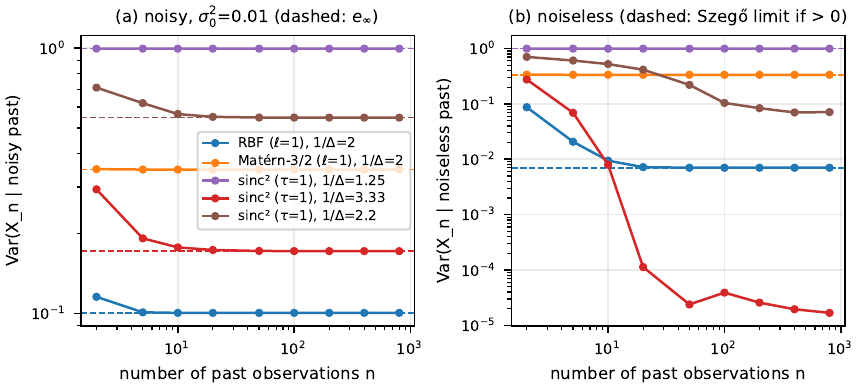}
    \caption{Exact finite-$n$ prediction errors of the sampled temporal process given its $n$ most recent (a)~noisy ($e_n$) and (b)~noiseless ($v_n$) values, computed from Toeplitz matrices. Dashed lines: asymptotic values of Lemma~\ref{lem:szego} and Remark~\ref{rem:noiseless}. In (b), $v_n > 0$ for every $n$, and the sinc$^2$ curves with $1/\Delta > 2$ converge to $0$ as $n \to \infty$ (Remark~\ref{rem:noiseless}). Their apparent floors (around $2 \cdot 10^{-5}$ for $1/\Delta = 3.33$ and $7 \cdot 10^{-2}$ for $1/\Delta = 2.2$) are numerical: they come from the pseudo-inverse used to condition on the ill-conditioned noiseless past.}
    \label{fig:prediction_convergence}
\end{figure}

\begin{figure}[h]
    \centering
    \includegraphics[width=0.95\textwidth]{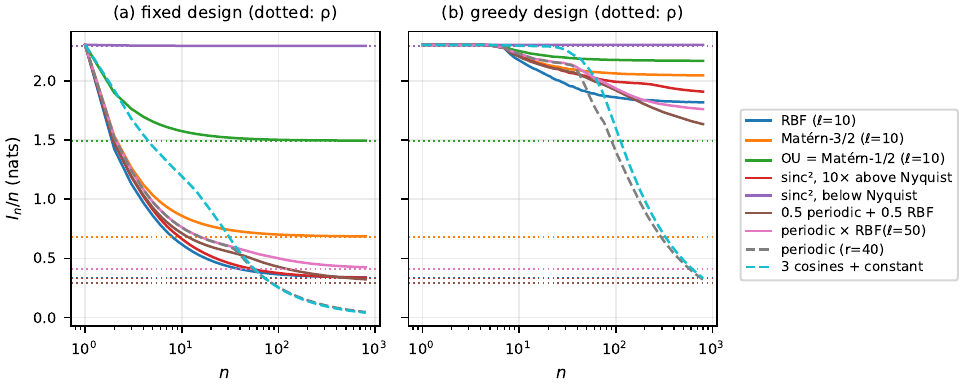}
    \caption{Information gain per step $I_n/n$, where $I_n = \frac{1}{2}\log\det\left(\bm I + \sigma_0^{-2}\bm K^{(n)}\right)$ is the mutual information between $f$ and the $n$ noisy observations of a given design (so that $\gamma_n(k)$ is the maximum of $I_n$ over all designs), of (a)~the fixed-location design of Proposition~\ref{prop:barrier}, which converges to $\rho$ (dotted lines) for kernels with a continuous spectral part, and (b)~the greedy design that queries the point of maximal posterior variance, a standard proxy for $\gamma_n(k)$. For rhythmic kernels (dashed lines), both vanish, as predicted by Lemma~\ref{lem:info_gain_atomic}.}
    \label{fig:info_gain}
\end{figure}

\begin{figure}[h]
    \centering
    \includegraphics[width=0.6\textwidth]{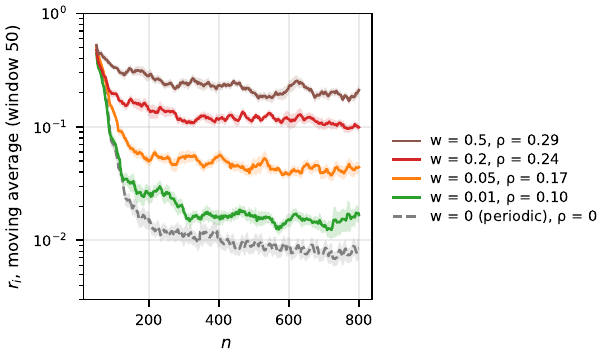}
    \caption{Moving average (window 50) of the instantaneous regret $r_i$ of GP-UCB on $(1-w)$ periodic $+\ w$ RBF ($\ell = 10$), for several weights $w$ of the broadband component (20 objectives per weight, $\pm$ one standard error), with the information rate $\rho$ of each mixture. Any $w > 0$ yields $\rho > 0$ and a plateau (Theorem~\ref{thm:lower_regret_bound}), whose height grows with $w$, hence with $\rho$. The mixture of Figure~\ref{fig:regret} corresponds to $w = 0.5$.}
    \label{fig:mixture}
\end{figure}

\begin{figure}[h]
    \centering
    \includegraphics[width=0.85\textwidth]{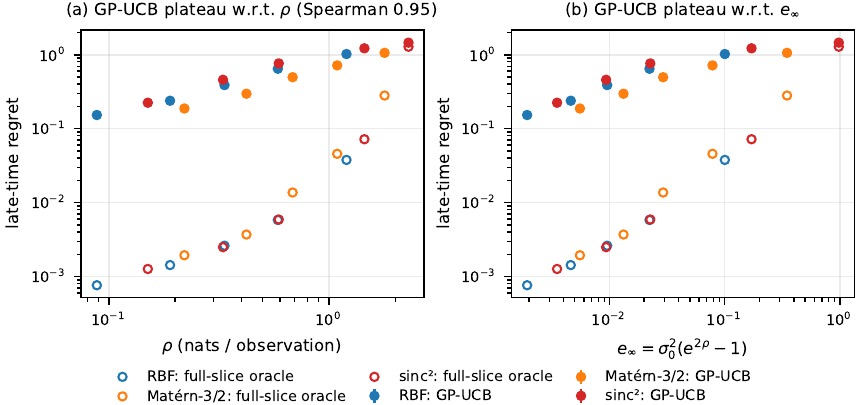}
    \caption{Late-time regret of GP-UCB (filled) and plateau of the full-slice oracle (hollow) for the 15 temporal kernels of Figure~\ref{fig:oracle} (right), as functions of (a)~$\rho$ and (b)~$e_\infty$. The plateaus of GP-UCB are ordered by $\rho$ (Spearman rank correlation $0.95$), but are up to $200$ times higher than those of the oracle for slowly varying kernels, and grow more slowly with $e_\infty$ (log-log slope of $0.46$, against $1.07$ for the oracle).}
    \label{fig:difficulty}
\end{figure}

\begin{figure}[h]
    \centering
    \includegraphics[width=0.85\textwidth]{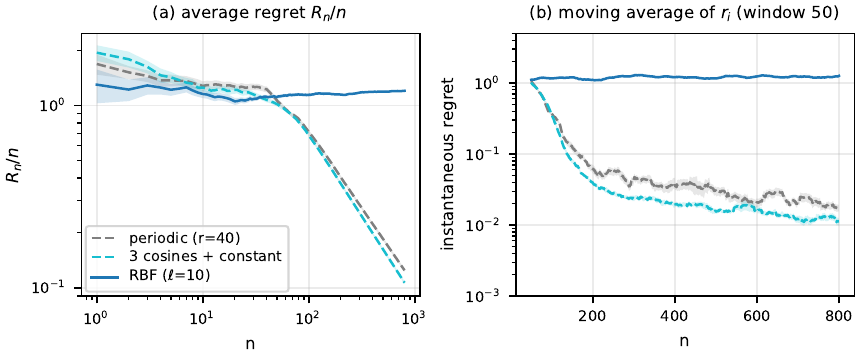}
    \caption{Regret of GP-UCB with the growing exploration parameter $\beta_i = 2\log(100\, i^2\pi^2/(6\delta))$, $\delta = 0.1$, of~\citet{gpucb}, i.e., the algorithm analyzed by Theorem~\ref{thm:upper_regret_bound} (in its finite-domain version, which holds on the grid), instead of the constant $\beta_i = 2$ of Figure~\ref{fig:regret} (20 objectives per kernel, $\pm$ one standard error). (a)~Average regret $R_n/n$. (b)~Moving average (window 50) of $r_i$. As predicted by Theorem~\ref{thm:upper_regret_bound}, the regret vanishes for the rhythmic kernels. It plateaus for the RBF kernel, at a late-time regret of $1.22$, closer to random search ($1.52$) than to GP-UCB with $\beta_i = 2$ ($0.39$, Table~\ref{tab:results}), because the large $\beta_i$ leads to near-random exploration.}
    \label{fig:theorem_d}
\end{figure}

\end{document}